%% file: main.tex
\documentclass{zgroup/zgroup}
\usepackage{times}

\usepackage{hyperref}
\usepackage{url}

\title[]{Theoretical bounds on estimation error for meta-learning}

\zgroupauthor{\Name{James Lucas} \Email{jlucas@cs.toronto.edu}\\
  \addr University of Toronto; Vector Institute\\
  \Name{Mengye Ren} \Email{mren@cs.toronto.edu}\\
  \addr University of Toronto; Vector Institute; Uber ATG\\
  \Name{Irene Kameni} \Email{ikameni@aimsammi.org}\\
  \addr African Master for Mathematical Sciences; Vector Institute\\
  \Name{Toniann Pitassi} \Email{toni@cs.toronto.edu}\\
  \addr University of Toronto; Vector Institute; Canadian Institute for Advanced Research\\
  \Name{Richard Zemel} \Email{zemel@cs.toronto.edu}\\
  \addr University of Toronto; Vector Institute; Canadian Institute for Advanced Research\\
}

\input{main/reqs}

\begin{document}

\maketitle

\input{main/abstract}
\input{main/introduction}

\input{main/related.tex}
\input{main/setting.tex}
\input{main/examples.tex}
\input{main/lower_bounds.tex}
\input{main/hierarchical_model.tex}

\input{main/experiments/experiments.tex}

\input{main/discussion.tex}

\input{main/acknowledgements.tex}

\bibliography{references}

\appendix
\onecolumn
\input{main/appendices/appendix.tex}
\end{document}

%% file: main/reqs.tex
\usepackage{microtype}
\usepackage{appendix}
\usepackage{caption}
\usepackage{graphicx, color}
\usepackage{booktabs}       
\usepackage{amsfonts}       
\usepackage{nicefrac}       
\usepackage{microtype}      
\usepackage{multirow, multicol}
\usepackage{ragged2e}
\usepackage{amsmath}
\usepackage{thmtools, thm-restate}
\usepackage{bm, bbm}
\usepackage{booktabs} 
\usepackage{enumitem}
\usepackage{tikz}
\usepackage{wrapfig}
\usetikzlibrary{positioning}
\usepackage{float}

\usepackage{hyperref}
\usepackage{cleveref}

\input{main/defs.tex}

%% file: main/defs.tex
\DeclareMathOperator*{\argmax}{\textrm{argmax}}

\newcommand{\iid}{\emph{i.i.d }}

\newcommand{\lossfn}{\ell}
\newcommand{\metricfn}{\rho}
\newcommand{\monofn}{\psi}

\newcommand{\estimator}{\hat{\theta}}

\newcommand{\loss}[3]{\lossfn(\estimator_{#1}(#2), \theta_{#3})}

\newcommand{\exploss}[4]{\bbE_{\substack{#3\sim #1^n\\\testS \sim #2^k}}\left[ \loss{#3}{#4}{#2} \right]}
\newcommand{\iidexploss}{\bbE_{S\sim P^k}\left[ \loss{}{S}{P} \right]}

\newcommand{\bx}{\mathbf{x}}
\newcommand{\by}{\mathbf{y}}

\newcommand{\envS}{S_{1:M}}
\newcommand{\testS}{S_{M+1}}

\newcommand{\bbE}{\mathbb{E}}

\newcommand{\argmin}{\mathop{\textrm{argmin}}}
\newcommand{\bbP}{\mathbb{P}}
\newcommand{\bbR}{\mathbb{R}}
\newcommand{\bbN}{\mathbb{N}}

\newcommand{\KL}[2]{D_{\mathrm{KL}}\left(#1\Vert#2\right)}

\newcommand{\calJ}{\mathcal{J}}
\newcommand{\calP}{\mathcal{P}}

\newcommand{\calV}{\mathcal{V}}
\newcommand{\calN}{\mathcal{N}}
\newcommand{\calF}{\mathcal{F}}
\newcommand{\bball}{\mathbb{B}}
\newcommand{\spaceZ}{\mathcal{Z}}
\newcommand{\spaceX}{\mathcal{X}}
\newcommand{\spaceY}{\mathcal{Y}}

\newcommand{\btheta}{\bm{\theta}}

\newcommand{\Covop}{\mathop{\mathrm{Cov}}}
\newcommand{\barP}{\bar{P}}
\newcommand{\norm}[1]{\lVert #1 \rVert}
\newcommand{\bignorm}[1]{\left\lVert #1 \right\rVert}

\newcommand{\bestimator}{\bm{\estimator}}
\newcommand{\bestimatorNovel}{\bm{\estimator}_{M+1}}
\newcommand{\lregparams}{\btheta}
\newcommand{\btau}{\bm{\tau}}
\newcommand{\bxi}{\bm{\xi}}
\newcommand{\beps}{\bm{\epsilon}}
\newcommand{\tr}{\mathrm{Tr}}
\newcommand{\sminX}{s_1}
\newcommand{\sminY}{s_2}

\newcommand{\posC}[1]{\tilde{C}_{#1}}
\newcommand{\posCX}{\posC{1}}
\newcommand{\nX}{{n}}
\newcommand{\nY}{{k}}
\newcommand{\DXi}[1]{X_{#1}}
\newcommand{\DX}{\DXi{}}
\newcommand{\DXP}{\DX \CTheta}
\newcommand{\DY}{\DXi{M+1}}
\newcommand{\Ci}[1]{\sigma_{#1}^2 I}
\newcommand{\CX}{\sigma_1^2 I}

\newcommand{\condX}{\kappa}
\newcommand{\condY}{\kappa_{M+1}}
\newcommand{\condtC}{\tilde{\kappa}}
\newcommand{\condtau}{\kappa_\tau}
\newcommand{\alphaX}{\alpha_1}
\newcommand{\alphaY}{\alpha_2}
\newcommand{\thetai}[1]{\bm{\theta}_{#1}}

\newcommand{\thetaY}{\thetai{M+1}}

\newcommand{\smax}{{s_{\max}}}
\newcommand{\smin}{{s_{\min}}}

\newcommand{\sCY}{\sigma_{M+1}}
\newcommand{\CTheta}{{\sigma_\theta^2 I}}
\newcommand{\smaxP}{{\sigma_\theta^2}}
\newcommand{\sminP}{{\sigma_\theta^2}}
\newcommand{\smaxCX}{{\sigma_1^2}}
\newcommand{\smaxCY}{{\sigma_{M+1}^2}}
\newcommand{\sminCX}{{\sigma_1^2}}

\newcommand{\XPXT}{\sigma_\theta^2 \DX \DX^\top}
\newcommand{\XPXTi}[1]{\sigma_\theta^2 \DXi{#1} \DXi{#1}^\top}
\newcommand{\XTCYX}{\sigma_{M+1}^{-2} \DY^\top \DY}
\newcommand{\XTCY}{\sigma_{M+1}^{-2} \DY^\top}
\newcommand{\sminCYInv}{\sCY^{-2}}
\newcommand{\CYInv}{\sCY^{-2}}
\newcommand{\CXs}{\sigma_1^2}

\newcommand{\CovPosTheta}{\Sigma_{\thetaY \vert \allData}}

\newcommand{\muPosTheta}{\mu_{\thetaY \vert \allData}}

\newcommand{\CovPosXThetaTau}{\Sigma_{\theta + \tau \vert \sourceData}}

\newcommand{\CovPosXTau}{\Sigma_{\tau \vert \sourceData}}

\newcommand{\muPosXTau}{\mu_{\tau \vert \sourceData}}

\newcommand{\sourceData}{Y_{1:M}}
\newcommand{\novelData}{\by_{M+1}}
\newcommand{\allData}{Y_{1:M+1}}
\newcommand{\PrecPosXTauExpanded}{\sum^{M}_{i=1}\DXi{i}^\top(\XPXTi{i} + \CXs I)^{-1}\DXi{i}}
\newcommand{\muPosXTauExpanded}{\CovPosXTau \sum^{M}_{i=1}\DXi{i}^\top (\XPXTi{i} + \CXs I)^{-1}\by_{i}}

\newcommand{\thetaEst}{\bestimatorNovel}
\newcommand{\CPos}{{\Sigma'}}
\newcommand{\CPosT}{{\Sigma'_{0}}}
\newcommand{\MConst}{A}
\newcommand{\MMConst}{A_1}
\newcommand{\MMMConst}{A_2}
\newcommand{\LConst}{L}
\newcommand{\LLConst}{L_1}
\newcommand{\LLLConst}{L_2}
\renewcommand{\dim}{d}
\newcommand{\cond}{\kappa}


%% file: main/abstract.tex
\begin{abstract}
    Machine learning models have traditionally been developed under the assumption that the training and test distributions match exactly. However, recent success in few-shot learning and related problems are encouraging signs that these models can be adapted to more realistic settings where train and test distributions differ. Unfortunately, there is severely limited theoretical support for these algorithms and little is known about the difficulty of these problems. In this work, we provide novel information-theoretic lower-bounds on minimax rates of convergence for algorithms that are trained on data from multiple sources and tested on novel data. Our bounds depend intuitively on the information shared between sources of data, and characterize the difficulty of learning in this setting for arbitrary algorithms. We demonstrate these bounds on a hierarchical Bayesian model of meta-learning, computing both upper and lower bounds on parameter estimation via maximum-a-posteriori inference.
\end{abstract}

%% file: main/introduction.tex
\vspace{0.3in}
\section{Introduction}
Many practical machine learning applications deal with distributional shift from training to testing. One example is few-shot classification
\citep{ravi2016optimization, vinyals2016matching}, where new classes need to be learned at test time based on only a few examples for each novel class. Recently, few-shot classification has seen increased success; however,
the theoretical properties of this problem remain poorly understood. 

In this paper we analyze the {\it meta-learning} setting, where 
the learner is given access to samples from a set of meta-training distributions, or tasks.  At test-time, the learner is exposed to only a small number of samples from some novel task. The meta-learner aims to uncover a useful inductive bias from the original samples, which allows them to learn a new task more efficiently.\footnote{Note that this definition encompasses few-shot learning.}
While some progress has been made towards understanding the generalization performance of specific meta-learning algorithms \citep{amit2017meta, khodak2019provable, pmlr-v98-bullins19a, pmlr-v97-denevi19a, cao2019theoretical}, little is known about the difficulty of the meta-learning problem in general.
Existing work has studied generalization upper-bounds for novel data distributions \citep{ben2010theory, amit2017meta}, yet to our knowledge, the inherent difficulty of these tasks relative to the \iid case has not been characterized. 

In this work, we derive novel bounds for meta learners. We first present a general information theoretic lower bound, Theorem~\ref{thm:env_lower_bound}, that we use to derive bounds in particular settings. Using this result, we derive lower bounds in terms of the number of training tasks, data per training task, and data available in a novel target task. Additionally, we provide a specialized analysis for the case where the space of learning tasks is only partially observed, proving that infinite training tasks or data per training task are insufficient to achieve zero minimax risk (Corollary~\ref{thm:env_lower_bound_asymp}).

We then derive upper and lower bounds for a particular meta-learning setting. In recent work, \citet{grant2018recasting} recast the popular meta-learning algorithm MAML \citep{finn2017model} in terms of inference in a Bayesian hierarchical model. Following this, we provide a theoretical analysis of a hierarchical Bayesian model for meta-linear-regression. We compute sample complexity bounds for posterior inference under Empirical Bayes \citep{robbins1956} in this model and compare them to our predicted lower-bounds in the minimax framework.  Furthermore, through asymptotic analysis of the error rate upper bound of the MAP estimator, we identify crucial features of the meta-learning environment which are necessary for novel task generalization.

\noindent Our primary contributions can be summarized as follows:
\begin{itemize}
    \item We introduce novel lower bounds on minimax risk of parameter estimation in meta-learning.
    \item Through these bounds, we compare the relative utility of samples from meta-training tasks and the novel task and emphasize the importance of the relationship between the tasks.
    \item We provide novel upper bounds on the error rate for estimation in a hierarchical meta-linear-regression problem, which we verify through an empirical evaluation.
\end{itemize}

%% file: main/related.tex
\section{Related work}

\citet{baxter2000model} introduced a formulation for inductive bias learning where the learner is embedded in an environment of multiple tasks. The learner must find a hypothesis space which enables good generalization on average tasks within the environment, using finite samples. In our setting, the learner is not explicitly tasked with finding a reduced hypothesis space but instead learns
using a two-stage approach, which matches the standard meta-learning paradigm \citep{vilalta2002perspective}. In the first stage an inductive bias is extracted from the data, and in the second stage the learner estimates using data from a novel task distribution. Further, we focus on bounding minimax risk of meta learners. Under minimax risk, an optimal learner achieves minimum error on the hardest learning problem in the environment. While average case risk of meta learners is more commonly studied, recent work has turned attention towards the minimax setting \citep{pmlr-v75-kpotufe18a, NIPS2019_9179, hanneke2020no}. The worst-case error 
in meta-learning is particularly important in safety-critical systems, for example in medical diagnosis.

There is a large volume of prior work studying upper-bounds on generalization error in multi-task environments \citep{ben2008notion, ben2010theory, pentina2014pac, amit2017meta}. While the approaches in these works vary, one common factor is the need to characterize task-relatedness. Broadly, these approaches either assume a shared distribution for sampling tasks \citep{baxter2000model, pentina2014pac, amit2017meta}, or a measure of distance between distributions \citep{ben2008notion, ben2010theory, mohri2012new}. Our lower-bounds utilize a weak form of task relatedness, assuming that the environment contains a finite set that is suitably separated in parameter space but close in KL divergence---this set of assumptions also arises often when computing \iid minimax lower bounds.

One practical approach to meta-learning
is learning a linear mapping on top of a learned feature space. 
For example, Prototypical Networks~\citep{snell2017prototypical} effectively learn a  discriminative embedding function and performs linear classification on top using the novel task data.
Analyzing these approaches is challenging due to metric-learning inspired objectives (that require non-\iid sampling) and the simultaneous learning of feature mappings and top-level linear functions yet
some progress has been made \citep{jin2009regularized, saunshi2019theoretical, wang2019multitask, du2020few}. \citet{maurer2009transfer}, for example, explores linear models fitted over a shared linear feature map in a Hilbert space. Our results can be applied in these settings if a suitable packing of the representation space is defined.

Other approaches to meta-learning aim to parameterize learning algorithms themselves. Traditionally, this has been achieved by hyper-parameter tuning~\citep{gpml,stn} but
recent fully parameterized optimizers also show promising performance in deep neural network optimization~\citep{l2l}, few-shot learning~\citep{ravi2016optimization}, unsupervised learning~\citep{metaunsup}, and reinforcement learning~\citep{duan2016rl2}. Yet another approach learns the initialization of task-specific parameters, that are further adapted through regular gradient descent. Model-Agnostic Meta-Learning \citep{finn2017model}, or MAML, augments the global parameters with a meta-initialization of the weight parameters. \citet{grant2018recasting} recast MAML in terms of inference in a Bayesian hierarchical model.
In Section~\ref{sec:hierarchical_bayes}, we consider learning in a hierarchical environment of linear models and provide both lower and upper bounds on the error of estimating the parameters of a novel linear regression problem.

Lower bounding estimation error is a critical component of understanding learning problems (and algorithms). Accordingly, there is a large body of literature producing such lower bounds \citep{khas1979lower, yang1999information, loh2017lower}. We focus on producing lower-bounds for parameter estimation using local packing sets, but expect that extending these results to density estimation or non-parametric estimation is feasible.

%% file: main/setting.tex
\section{Novel task environment risk}
\label{sec:minimax_setting}
Most existing theoretical work studying out-of-distribution generalization focuses on providing upper-bounds on generalization performance \citep{ben2010theory, pentina2014pac, amit2017meta}. We begin by instead exploring the converse: what is the best performance we can hope to achieve on any given task in the environment? After introducing notation and minimax risks, we then show how these ideas can be applied, using meta linear regression as an example.

A full reference table for notation can be found in Appendix~\ref{app:notation} and a short summary is given here.
We consider algorithms that learn in an environment $(\spaceZ, \calP)$, with data domain $\spaceZ = \spaceX \times \spaceY$ and $\calP$ a space of distributions with support $\spaceZ$. In the typical \iid setting, the algorithm is provided training data $S \in \spaceZ^k$, consisting of $k$ \iid samples from $P \in \calP$. 

In the standard {\it multi-task} setting, we sample training data from a set of training tasks\\$\{P_1,\ldots,P_{M+1}\} \subset \calP$.
We extend this to a meta-learning, or {\it novel-task} setting by first drawing $\envS$: $n$ training data points from the first $M$ distributions, for a total of $nM$ samples. We call this the {\it meta-training set}. We then draw a small sample of novel data,
called a {\it support set}, $\testS \in \spaceZ^k$, from $P_{M+1}$. 

Consider a symmetric loss function $\lossfn(a,b) = \monofn(\metricfn(a, b))$ for non-decreasing $\monofn$ and arbitrary metric $\metricfn$. We seek to estimate the output of $\theta: \calP \rightarrow \Omega$, a functional that maps distributions to a metric space $\Omega$. For example, $\theta(P)$ may describe the coefficient vector of a high-dimensional hyperplane when $\calP$ is a space of linear models, and $\metricfn$ may be the Euclidean distance.

\paragraph{The \iid minimax risk} Before studying the meta-learning setting, we first begin with a definition of the \iid minimax risk that measures the worst-case error of the best possible estimator,
\begin{equation}\label{eqn:iid_risk}
    R^* = \inf_{\estimator}\sup_{P \in \calP}\iidexploss.
\end{equation}
For notational convenience, we denote the output of $\theta(P)$ by $\theta_P$. The estimator for $\theta$ is denoted, $\estimator: \spaceZ^k \rightarrow \Omega$, and maps $k$ samples from $P$ to an estimate of $\theta_P$.

\paragraph{Novel-task minimax risk} In the novel-task setting, we wish to estimate $\theta_{P_{M+1}}$, the parameters of the novel task distribution $P_{M+1}$.
We consider two-stage estimators for $\theta_{P_{M+1}}$. In the first stage, the meta-learner uses a learning algorithm $f: \envS \mapsto \bestimator_{\envS}$,  that maps the meta-training set to an estimation algorithm, $\bestimator_{\envS}:~\spaceZ^k\rightarrow~\Omega$. In the second stage, the learner computes $\bestimator_{\envS}(\testS)$, the estimate of $\theta_{P_{M+1}}$.

\noindent The novel-task minimax risk is given by, 
\begin{equation}\label{eqn:novel_risk}
    R^*_\calP = \inf_{\estimator} \sup_{P_1,\ldots,P_{M+1} \in \calP} \exploss{P_{1:M}}{P_{M+1}}{\envS}{\testS}
\end{equation}

The estimator for $\theta_{M+1}$ now depends additionally on the $Mn$ samples in $\envS$, where only $k\ll~Mn$ samples from $P_{M+1}$ are available to the learner. This quantity addresses the domain shift expected at test-time in the meta-learning setting and allows the learner to use data from multiple tasks.

The goal of $f$ in this setting is to learn an inductive bias from $\envS$ such that useful inferences can be made with only $k$ data points from the novel distribution, $P_{M+1}$. In this setting, $k$ is equivalent to the number of shots in the few-shot learning setup.

%% file: main/examples.tex
\paragraph{An example with meta-linear regression}
We present here a short summary based on meta linear regression, which we will analyze
in more detail in Section~\ref{sec:hierarchical_bayes}.

\begin{figure}
\centering
\includegraphics[width=\linewidth]{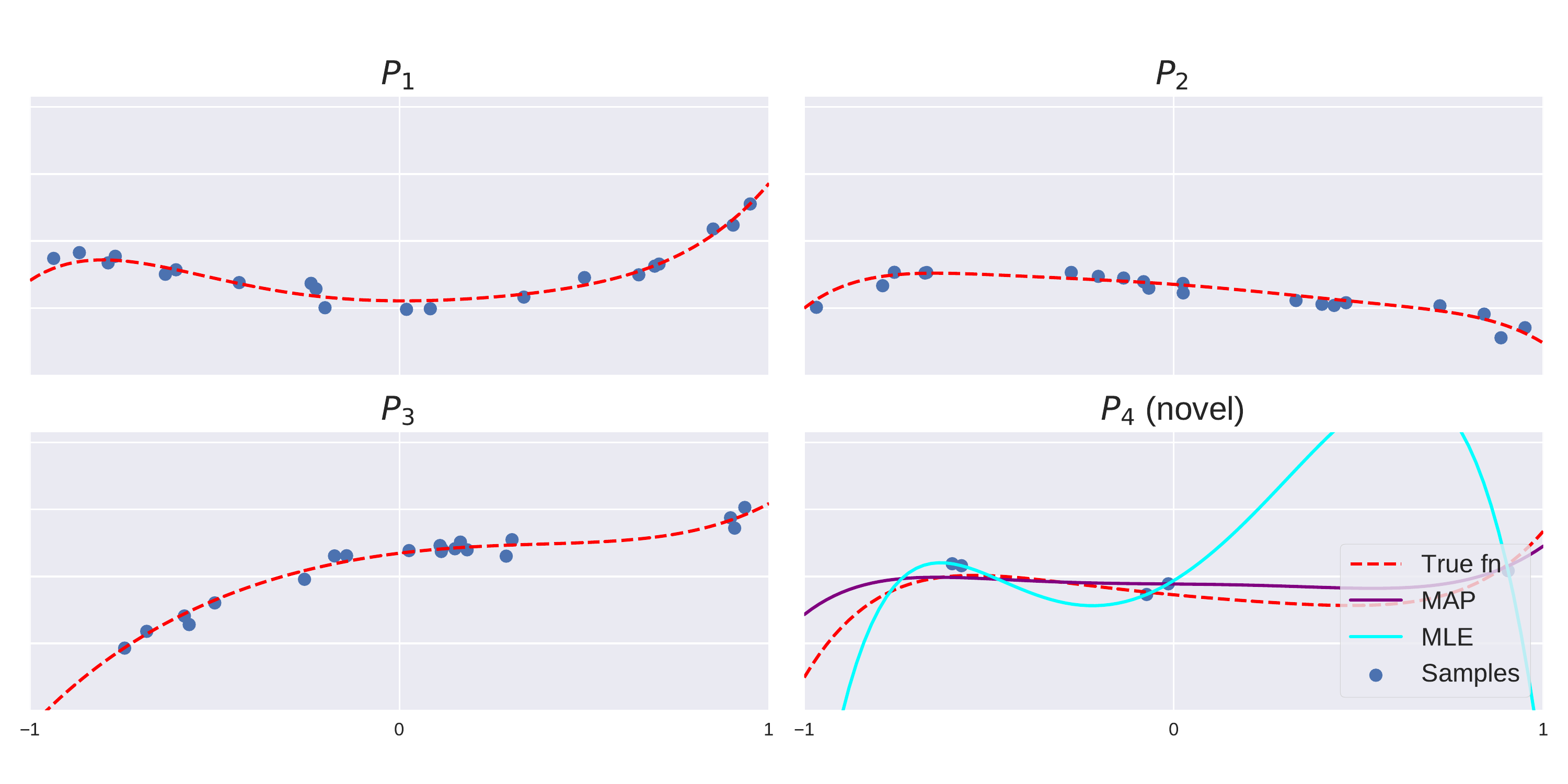}%
\caption{\textbf{Meta-learning 1D-regression:} The parameters of a 1D regression model are fitted
from a small support set. The training distributions ($P_1,P_2,P_3$) give a useful inductive bias
for fitting $P_4$ using only 5 points. The MLE solution on the novel task for those 5 points is also
displayed.}
\label{fig:meta_lin_ref}
\end{figure}
In Figure~\ref{fig:meta_lin_ref}, we show observed data samples from a family of polynomial
regression models. Our aim is to output an algorithm which recovers the parameters of a new polynomial function from limited observations--we choose a MAP estimator which is described fully in
Section~\ref{sec:hierarchical_bayes}. In the bottom right, we are given only 5 data points from a
novel task distribution and estimate the parameters of the model with both the MLE and MAP
estimators --- the MLE overfits the support set while the MAP estimator is close to the true function.

\noindent In terms of the terminology used above, the set,
\[\calP = \{p_{\lregparams}(y)=\calN(\bx^\top \lregparams,\sigma^2): \lregparams \in \bbR^{d}, \bx=[1,x,\ldots,x^{d-1}]\},\]
is the space of polynomial regression models, parameterized by $\lregparams$. For this problem, we take $\lossfn(\estimator, \theta) = \norm{\estimator - \theta}_2^2$. The tasks are generated with $p(\theta) = \calN(\tau, \sigma^2_\theta)$, for unknown, sparse, $\tau \in \bbR^d$. Thus, each model is a polynomial function with few large coefficients. The algorithm $f$, first takes samples from $P_1,P_2,P_3$ and computes an estimate, $\hat{\tau}$. This estimate of $\tau$ is then used to compute $\bestimator(\testS; \hat{\tau}) = \argmax_{\lregparams_4} p(\lregparams_4|\hat{\tau}, \testS)$. Note that this approach is able to learn the correct inductive bias from the data, without requiring a carefully designed regularizer.  The lower bounds we derive in Section~4 can be applied to problems of this general type, and the upper and lower bounds in Section~5 apply specifically to this setting.

%% file: main/lower_bounds.tex
\section{Information theoretic lower bounds on novel task generalization}
\label{sec:lbounds}
\vspace{-0.1in}

In this section, we first present our most general result: Theorem~\ref{thm:env_lower_bound}. Using this, we derive Corollary~\ref{corollary:env_lbound_packing} that gives a lower bound in terms of the sample size in the training and novel tasks. Corollary~\ref{corollary:env_lbound_packing} recovers the well-known \iid lower bound (Theorem~\ref{thm:lower_bound}) when $Mn=0$, and, importantly, highlights that the novel task data is significantly more valuable than the training task data. Additionally, we provide a specialized bound that applies when the environment is partially observed --- proving that in this setting training task data is insufficient to drive the minimax risk to zero.

In Theorem~\ref{thm:env_lower_bound}, we assume that $\calP$ contains $J$ distinct $2\delta$-separated distributions but the learner observes data from only $M+1 \leq J$ of them. Intuitively, the learner observes the environment and their error rate lower-bound shrinks as the amount of information shared between the training tasks and the novel task grows.
All proofs are given in Appendix~\ref{app:proofs:lbounds}. Recall $\lossfn(a,b) = \monofn(\metricfn(a, b))$ for non-decreasing $\monofn$ and arbitrary metric $\metricfn$.

\begin{restatable}[Minimax novel task risk lower bound]{theorem}{envlbound}\label{thm:env_lower_bound}
Let $\calJ \subset \calP$ contain $J$ distinct distributions such that $\metricfn(\theta_{P}, \theta_{P'})~\geq~2\delta$ for all $P,P' \in \calJ$. 
Let $\pi$ be a random ordering of the $J$ elements, and $Z|\pi$ be a vector of $k$ \iid samples from $P_{\pi_{M+1}}$.
Further, define $W|\pi$ to be an $n \times M$ matrix whose $j^{th}$ column consist of $n$ \iid samples from $P_{\pi_j}$. Then,
\begin{align*}
R^*_\calP \geq
\monofn(\delta)\left(1 - \dfrac{I(\pi_{M+1}; W) + I(\pi_{M+1};Z) + 1}{\log_2 J}\right).
\end{align*}
\end{restatable}
 
To derive this result, we bound the statistical estimation error by the error on a corresponding decoding problem where we must predict the novel task index, given the meta-training set $\envS$ and $\testS$. Fano's inequality provides best-case error probabilities for this problem.

Using Theorem~\ref{thm:env_lower_bound}, we derive our first bound on the novel-task minimax risk that depends on the number of training tasks and datapoints per training task. The following corollary implies that if $J$ of the previous meta-training tasks are close
(in terms of their pairwise KL distance), then learning a novel task from training samples drawn
from the meta-training tasks requires significantly more examples; in particular, learning the novel task from samples drawn from the
meta training set requires  $\Omega(J)$ times the sample complexity of the novel task. This matches our intuition that learning the novel task implies the ability to distinguish it from all $J$ well-separated meta-training tasks.
\begin{corollary}\label{corollary:env_lbound_packing}
Assume the same setting as in Theorem~\ref{thm:env_lower_bound} and additionally that $\KL{P_i}{P_j} \leq \alpha$ for all $P_i,P_j \in \calJ,\:i \neq j$. Then,
\begin{align*}
R^*_\calP \geq \monofn(\delta) \left(1 - \dfrac{1 + (\frac{Mn}{J-1}+k)\alpha}{\log_2 J}\right).
\end{align*}
\end{corollary}
Typically, practical instances of this bound require $\monofn(\delta) = O(1/k)$ or similar, as in Theorem~\ref{thm:mlreg_lower_bound} below.

\paragraph{A tighter bound on partially observed environments} We now consider the special case of Theorem~\ref{thm:env_lower_bound} when $M < J-1$, meaning that the meta-training tasks cannot cover the full environment. In this setting, we prove that no algorithm can generalize perfectly to tasks in unseen regions of the space with small $k$, regardless of the number of data points $n$ observed in each meta-training task.

\begin{restatable}[]{corollary}{envlboundasymp}\label{thm:env_lower_bound_asymp}
Consider the setting given in Corollary~\ref{corollary:env_lbound_packing}, with $M+1 < J$. Then,
\begin{align*}
R^*_\calP \geq \monofn(\delta)\left(\dfrac{\log_2 (J - M) - k\alpha - 1}{\log_2 J}\right).
\end{align*}
\end{restatable}

In this work, we have focused on the setting where $W$ contains an equal number of samples from each of the meta-training tasks --- this is the sampling scheme shown in Figure~2. However, it is possible to extend these results to different sampling schemes for $W$. For example, in the appendix we derive bounds with $W|\pi$ as a mixture distribution. Surprisingly, despite the task identity being hidden from the learner, the asymptotic rate for these two sampling schemes match.

\subsection{Measuring Task-Relatedness}

The use of local packing requires the design of an appropriate set of distributions whose corresponding parameters are $2\delta$-separated but maintain small KL divergences. In the multi-task setting such an assumption is intuitively reasonable: challenging tasks should require separated parameters for ideal explanations ($2\delta$-separated) but should satisfy some relatedness measure (small KL). As we will see shortly, lower bounds on minimax risk in the \iid setting may also assume the same notion of relatedness between the distributions in $\calP$.

Task relatedness is a necessary feature for upper-bounds on novel task generalization, but is typically difficult to define (see e.g. \citet{ben2008notion}). Our lower bounds utilize a relatively weak notion of task-relatedness, and thus may be overly pessimistic compared to the upper bounds computed in existing work. Note however that task relatedness of the form utilized here can be formulated in a representation space shared across tasks and thus can be applied in settings like those explored by e.g. \citet{du2020few}. Deriving lower bounds under other
task relatedness assumptions present in the literature would make for exciting future work.

\subsection{Comparison to risk of \iid learners}

From the statement of Theorem~\ref{thm:env_lower_bound} it is not clear how this lower-bound compares to that of the \iid learner which has access only to the $k$ samples from $\testS$.
To investigate the benefit of additional meta-training tasks, we compare our derived minimax risk lower bounds to those achieved by \iid learners. To do so, we revisit a standard result on minimax lower bounds that can be found in e.g. \citet{loh2017lower}. 

\begin{restatable}[IID minimax lower-bound]{theorem}{iidlbound}\label{thm:lower_bound}
Suppose $\{P_1,\dots,P_J\} \subseteq \calP$ satisfy $\metricfn(\theta_{P_i}, \theta_{P_j}) \geq 2\delta$ for all $i\neq j$. Additionally assume that $\KL{P_i}{P_j} \leq \alpha,$ for all pairs $i$ and $j$, then,
\begin{align*}
R^* \geq \monofn(\delta)\left(1 - \dfrac{k\alpha + 1}{\log_2 J}\right).
\end{align*}
\end{restatable}

We include a standard proof of this result in Appendix~\ref{app:proofs:lbounds}. As hoped, Corollary~\ref{corollary:env_lbound_packing} recovers Theorem~\ref{thm:lower_bound} when there are no training tasks available. Moreover, this \iid bound is strictly larger than the one computed in Corollary~\ref{corollary:env_lbound_packing} in general. Note that while this \iid minimax risk is asymptotically tight for several learning problems \citep{loh2017lower, raskutti2011minimax}, there is no immediate guarantee that the same is true for our meta-learning minimax bounds. We investigate the quality of these bounds by providing comparable upper bounds in the next section.

%% file: main/hierarchical_model.tex

\section{Analysis of a hierarchical Bayesian model of meta-learning}
\label{sec:hierarchical_bayes}

Our goal is to analyze the sample complexity of learning in a 
simple yet popular and practical model, where samples are drawn from multiple meta-training tasks and we want to generalize to a new task with only a few data points. After introducing the model, we will compute lower-bounds on the minimax risk using our results from Section~\ref{sec:lbounds}, revealing a $2^{d}$ scaling on the meta-training sample complexity. Following the lower bound, we derive an accompanying upper-bound on the risk of a MAP estimator. Asymptotic analysis of this bound reveals that if the observed samples from the novel task vary considerably more than the task parameters, then observing more meta-training samples may significantly improve convergence in the small $k$ regime.

\begin{minipage}[t]{0.48\linewidth}
For $i = 1...M+1$, where $M+1$ is the total number of tasks, we define,
\begin{align*}
\by_i      &= X_i\lregparams_i + \beps_i, &&  X_i \in \bbR^{n_i \times d}, \by_i \in \bbR^{n_i}, \beps_i \in \bbR^{n_i} \\
\beps_i &\sim \mathcal{N}(0, \sigma^2_i I), && \sigma^2_i \in \bbR^{+} \\
\lregparams_i &= \btau + \bxi, && \btau \in \bbR^{d}, \bxi \in \bbR^{d} \\
\bxi &\sim \mathcal{N}(0, \sigma^2_\theta I), && \sigma^2_\theta \in \bbR^{+}
\end{align*}
Each task has some design matrix $X_i$ and unknown parameters $\lregparams_i$. For simplicity, we assume known isotropic noise models and that $n_i = \nX$ for all $i\leq M$, with $n_{M+1}=\nY$.
\end{minipage}\hfill%
\begin{minipage}[t]{0.48\linewidth}
\centering
    \strut\vspace*{-\baselineskip}\newline
    \begin{tikzpicture}[
    roundnode/.style={circle, draw=gray, very thick, minimum size=8mm},
    observednode/.style={circle, draw=gray, fill=gray!25, very thick, inner sep=2pt},
    ]
    \node[roundnode]        (tau)                              {$\tau$};
    \node[roundnode]        (thetax)       [below left=0.8cm and 1.2cm of tau] {$\btheta_1$};
    \node[roundnode]        (thetay)       [below right=0.8cm and 1.2cm of tau] {$\btheta_2$};
    \node[observednode]     (x1)       [below left=0.8cm and 0.5cm of thetax] {$\by^{(1)}_1$};
    \node[observednode]     (xn)       [below right=0.8cm and 0.5cm of thetax] {$\by^{(n)}_1$};
    \node[observednode]     (y1)       [below left=0.8cm and 0.5cm of thetay] {$\by^{(1)}_2$};
    \node[observednode]     (yk)       [below right=0.8cm and 0.5cm of thetay] {$\by^{(k)}_2$};
     
    \draw[->] (tau) -- (thetax);
    \draw[->] (tau) -- (thetay);
    \draw[->] (thetax) -- (x1.north);
    \path (x1) -- node[auto=false]{\ldots} (xn);
    \draw[->] (thetax) -- (xn.north);
    \draw[->] (thetay) -- (y1.north);
    \path (y1) -- node[auto=false]{\ldots} (yk);
    \draw[->] (thetay) -- (yk.north);
    
    \end{tikzpicture}
    
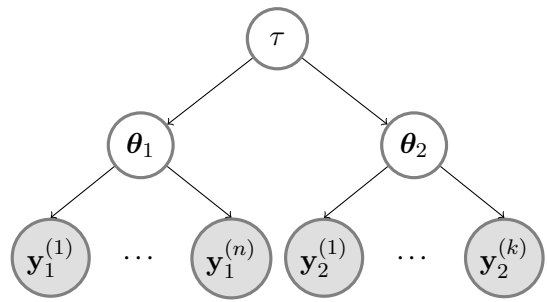
\captionof{figure}{A simple two-task hierarchical model.}
    \label{fig:hierachical_gaussian_model}
\end{minipage}

\noindent We will consider the Maximum a Posterior estimator,
\[\bestimatorNovel = \argmax_{\thetaY}p(\thetaY|\by_1,\ldots,\by_{M+1}),\]
and will characterize its risk, $\bbE[\norm{\bestimatorNovel~-~\thetaY}_2^2]$.

We can derive the posterior distribution under the Empirical Bayes estimate for $\btau$, which we defer to Appendix~\ref{app:proofs:ubounds}. The derivation is standard but dense and we recommend dedicated readers to consult \citet{gelman2013bayesian}, or an equivalent text, for more details.

\subsection{Minimax lower bounds}

We now compute lower bounds for parameter estimation with meta-learning over multiple linear regression tasks.
Beginning with a definition of the space of data generating distributions,
\[\calP_{LR} = \{ p_{\lregparams}(\by) = \calN(X\lregparams, \sigma^2 I): \lregparams \in \bball_2(1), X \in \bbR^{n \times d} \}, \]

where $\lregparams$ are the parameters to be learned, and $X$ is the design matrix of each linear regression task in the environment.
We write $\gamma = \max_{i} \sigma_{\max}(X_i/\sqrt{n})$, which we assume is bounded for all $X$ and $n$ (an assumption that is validated for random Gaussian matrices by \citet{raskutti2011minimax}).

\begin{restatable}[Meta linear regression lower bound]{theorem}{mlreglbound}\label{thm:mlreg_lower_bound}
Consider $\calP_{LR}$ defined as above and let $\lossfn(a,b) = (\norm{a-b}_2)^2$. If $d > 2$, then,
\[R^*_{\calP_{LR}} \geq O\left(\frac{d\sigma^2}{\gamma^2 (2^{-d}\nX M + \nY)}\right) \]
\end{restatable}

The proof is given in Appendix~\ref{app:proofs:linear_lower}. We see that the size of the meta-training set has an inverse exponential scaling in the dimension, $d$. This reflects the complexity of the space growing exponentially in dimensions and the need for a matching growth in data size to cover the environment sufficiently.

\subsection{Minimax upper bounds}

To compute upper bounds on the estimation error, we require an additional assumption. Namely, we will assume that the design matrices also have bounded minimum singular values, $0 < s \leq \sigma_{\min}(X/\sqrt{n})$ (see \citet{raskutti2011minimax} for some justification). For the upper-bounds, we allow the bounds on the singular values of the design matrices and the observation noise in the novel task to be different than those in the meta-training tasks. We note that we can still recover the setting assumed in the lower bounds, where all tasks match on these parameters, as a special case.

We consider the setting where the learner observes $\nX$ data points from each linear regression model in $\{ P_{\lregparams_1}, \ldots, P_{\lregparams_M} \} \subset \calP$. We then bound the error of estimating the parameters of some new model, $P_{\lregparams_{M+1}}$, of which $\nY$ samples are available.

The expected error rate of the MAP estimator can be decomposed as the posterior variance and bias squared. In the appendix we provide a detailed derivation of these results. The bound depends on dimensionality $d$, the observation noise in each task $\sigma^2_i$, the number of tasks $M$, the number of data points in each meta-training task $n$, and the number of data points in the novel task $k$.

\begin{restatable}[Meta Linear Regression Upper Bound]{theorem}{mlrbv}\label{thm:lregression_bias_variance}
Let $\thetaEst$ be the maximum-a-posteriori estimator, $\muPosTheta$. Then,
\[ R^*_{\calP_{LR}} \leq \sup_{\btheta_1,\ldots,\thetaY \in \bball_2(1)}\bbE[\Vert\thetaEst - \thetaY\Vert^2] \leq O\Bigl(\dim \sigma^2_{M+1} C(M, \nX, \nY)^{-2} D(M, \nX, \nY)\Bigr) \]
where,
\[C(M, \nX, \nY) = \left[ \nY + 
\frac{M\nX}{\frac{\nX (M +\condX^2)\sminY^2}{\alphaY} + A} \right], \textrm{ and,  }D(M, \nX, \nY) = \left[ \nY + \frac{M\nX}{ 
(\frac{\nX}{\LLConst} + \MMConst ) 
(\frac{M\nX}{\LLLConst} + \MMMConst)}\right].\]
Expectations are taken over the data conditioned on $\lregparams_1,\ldots,\lregparams_{M+1}$. Additional terms not depending on $d,\:M,\:\nX,\:\nY$ are defined in Appendix~\ref{app:proofs:ubounds}.
\end{restatable}

While the bounds presented in Theorem~\ref{thm:lregression_bias_variance} are relatively complicated, we can probe the asymptotic convergence of the MAP estimator to the true task parameters, $\thetaY$. In the following section, we will discuss some of the consequences of this result and its implications for our lower bounds.

\subsection{Asymptotic behavior of the MAP estimator} 

We first notice that when $k$ 
(the number of novel task examples) is small, the risk cannot be reduced to zero by adding more meta-training data. Recent work has suggested such a relationship may be inevitable \citep{hanneke2020no}. Our lower bound presented in Corollary~\ref{thm:env_lower_bound_asymp} agrees that more samples from a small number of meta-training tasks will not reduce the error to zero. However, 
unlike our lower bounds
based on local packing, the lower bounds presented in this section predict that if the meta-training tasks cover the space sufficiently then an optimal algorithm might hope to reduce the error entirely with enough samples. We hypothesize that this gap is due to limitations in the standard proof techniques we utilize for the lower-bounds when the number of tasks grows, and expect a sharper bound may be possible.

To emulate the few-shot learning setting where $k$ is relatively small, we consider $n \rightarrow \infty$, with $k$ and $M$ fixed. In this case, the risk is bounded as,
\[\sup_{\theta_1,\ldots,\thetaY \in \bball_2(1)}\bbE[\Vert\thetaEst - \thetaY\Vert^2] \leq O\left(\dim \sigma^2_{M+1}\left[k + \frac{2\alpha_2 M}{M + \condX^2} \right]^{-1}\right),\]
where $\alpha_2 = \smaxCY / \smaxP$, is the ratio of the observation noise to the variance in sampling $\btheta$, and $\condX$ is the condition number of the design matrices. This leads to a key takeaway: if the observed samples from $P_{M+1}$ vary considerably more than the parameters $\btheta$, then observing more samples in $\envS$ will significantly improve convergence towards the true parameters in the small $k$ regime. Further, adding more tasks (increasing $M$) also improves these constant factors by removing the dependence on the condition number, $\kappa$.

%% file: main/experiments/experiments.tex
\section{Empirical Investigations}
\label{sec:empirical}

\input{main/experiments/hierarchical_gaussian.tex}

\input{main/experiments/sinusoids.tex}

%% file: main/experiments/hierarchical_gaussian.tex
In this section, we provide additional quantitative exploration of the upper bound studied in Section~\ref{sec:hierarchical_bayes}. 
The aim is to take steps towards relating the bounds to experimental results;
we know of little theoretical work in meta-learning that attempt to relate their results to
practical empirical datasets.

\subsection{Hierarchical Bayes polynomial regression}

We first focus on the setting of polynomial regression over inputs in the range $[-1,1]$. Some examples of these functions and samples are presented in Figure~\ref{fig:meta_lin_ref}, alongside the MAP and MLE estimates for the novel task. Full details of the data used can be found in Appendix~\ref{app:exp_details}.


Figure~\ref{fig:hierarchical_lreg_simulation} shows the analytical expected error rate (risk) under various environment settings. We observe that even in this simple hierarchical model, the estimator exhibits complex behavior that is correctly predicted by Theorem~\ref{thm:lregression_bias_variance}. In Figure~\ref{fig:hierarchical_lreg_simulation}A, we varied the novel task difficulty by increasing the novel task observation noise ($\sigma^2_{M+1}$). We plot three curves for three different dataset size configurations. When the novel task is much noisier than the source tasks, it is greatly beneficial to add more meta-training data (blue vs. red). And while larger $k$ made little difference when the novel task was relatively difficult (blue vs. green), the expected loss was orders of magnitude lower when the novel task became easier. In Figure~\ref{fig:hierarchical_lreg_simulation}B, we fixed the relative task difficulty and instead varied $k$ and $M$. The x-axis now indicates the total data $Mn + k$ available to the learner. We observed that adding more tasks has a large effect in the low-data regime but, as predicted, the error has a non-zero asymptotic lower-bound --- eventually it is more beneficial to add more novel-task data samples.

These empirical simulations verify that our theoretical analysis is predictive of the behavior of this meta learning algorithm, as each of these observations can be inferred from Theorem~\ref{thm:lregression_bias_variance}. While this model is simple, it captures key features of and provides insight into the more general meta-learning problem.

\begin{figure}
\centering
\begin{minipage}{0.5\linewidth}%
    \centering%
    \includegraphics[width=\linewidth]{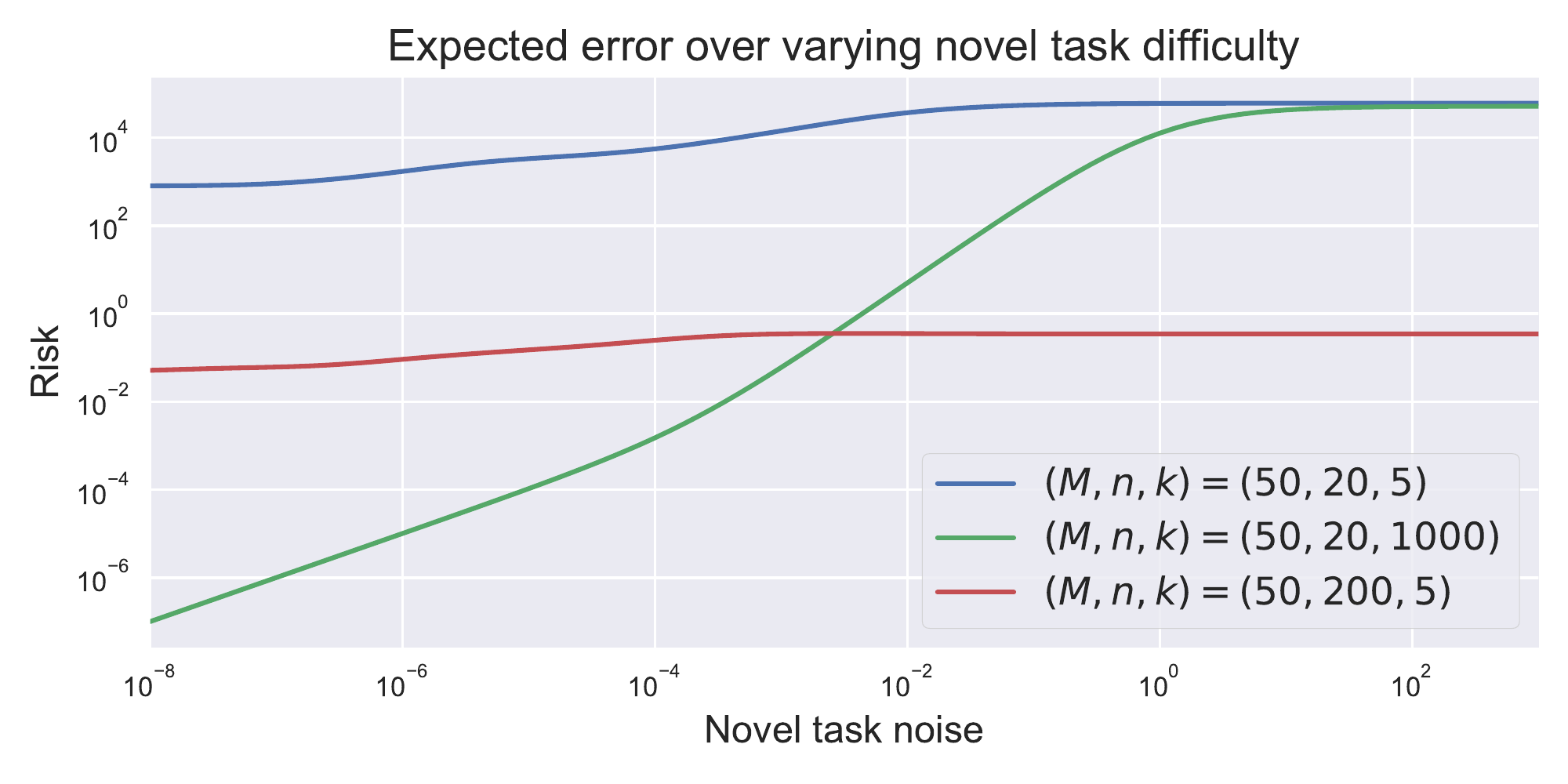}
    (A)
\end{minipage}%
\begin{minipage}{0.5\linewidth}%
    \centering%
    \includegraphics[width=\linewidth]{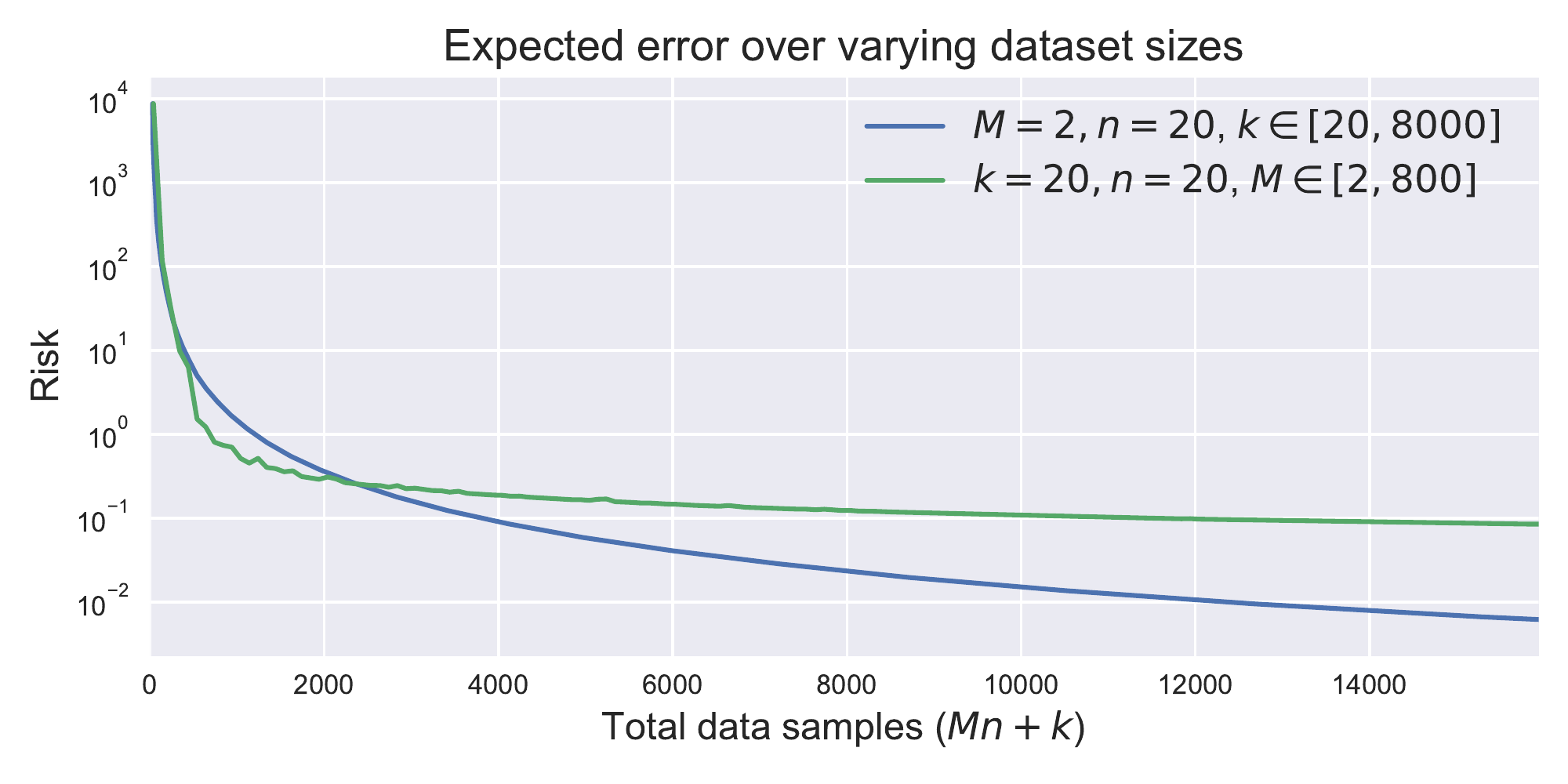}
    (B)
\end{minipage}
\caption{The expected error rate of the hierarchical MAP estimator, $\thetaEst$, over different environment
hyperparameter settings. \textbf{A)} The novel task observation noise is increased, making the novel task harder to learn. \textbf{B)} We increase the size of the dataset, in one case adding new tasks ($M$) and in the other adding new novel task data samples ($k$).}
\label{fig:hierarchical_lreg_simulation}
\end{figure}

%% file: main/experiments/sinusoids.tex
\subsection{Sinusoid regression with MAML}

Following the connections between MAML and hierarchical Bayes explored by \citet{grant2018recasting}, we also explored regression on sinusoids using MAML. Our aim was to investigate how predictive our linear theory is for this highly non-linear problem setting. As in \citet{finn2017model}, we sample sinusoid functions by placing a prior over the amplitude and phase. In other works \citep{finn2017model, grant2018recasting} the same prior is used for the training and testing stages. However, to better measure generalization to novel tasks we use different prior distributions when training versus evaluating the model. Full details of the experimental set-up and data sampling procedure can be found in Appendix~\ref{app:exp_details}.

We display the risk averaged over 30 trials in Figure~\ref{fig:maml_sinusoid}. We varied the novel task difficulty by increasing the observation noise in the novel task.
We plot separate curves for different dataset size configurations, and observe that the empirical results align fairly well with the results derived by sampling the hierarchical model (Figure \ref{fig:hierarchical_lreg_simulation}A). Adding more meta-training data (increasing $n$) is beneficial (green vs. yellow)
and adding more test data-points (higher $k$) is also beneficial (red vs. green).  Here however, these relationships did not interact with the task difficulty, as the wins for increased meta-training and meta-testing data were consistent, until task noise prevents any setting of the model from performing the task.

\begin{wrapfigure}{r}{0.48\textwidth}
  \begin{center}
    \includegraphics[width=0.48\textwidth]{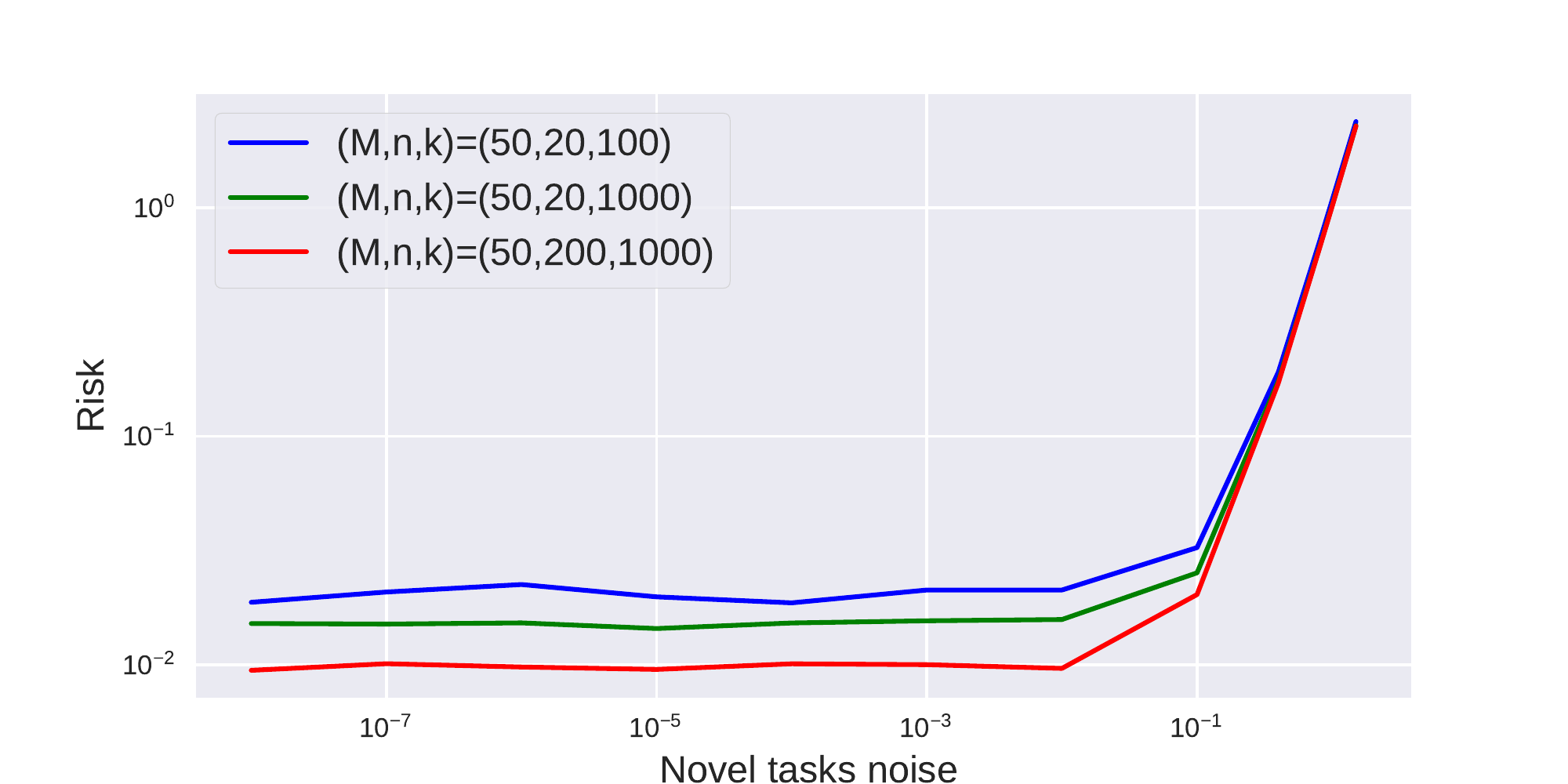}
  \end{center}
  \caption{Average risk for regressing sinusoid functions with MAML.}
  \label{fig:maml_sinusoid}
\end{wrapfigure}

%% file: main/discussion.tex
\section{Conclusion}
\vspace{-0.1in}
Meta-learning algorithms identify the inductive bias from source tasks and make models more adaptive towards unseen novel distribution.
In this paper, we take initial steps towards characterizing the difficulty of meta-learning and understanding how these limitations present in practice. We have derived both lower bounds and upper bounds on the error of meta-learners, which are particularly relevant in the few-shot learning setting where $k$ is small. Our bounds capture key features of the meta-learning problem, such as the effect of increasing the number of shots or training tasks.
We have also identified a gap between our lower and upper bounds when there are a large number of training tasks, which we hypothesize is a limitation of the proof technique that we applied to derive the lower bounds --- suggesting an exciting direction for future research.

%% file: main/acknowledgements.tex
\section{Acknowledgements}

This work benefited greatly from the input of many other researchers. In particular, we extend our thanks to Shai Ben-David, Karolina Dziugaite, Samory Kpotufe, and Daniel Roy for discussions and feedback on the results presented in this work. We also thank Elliot Creager, Will Grathwohl, Mufan Li, and many of our other colleagues at the Vector Institute for feedback that greatly improved the presentation of this work. Resources used in preparing this research were provided, in part, by the Province of Ontario, the Government of Canada through CIFAR, and companies sponsoring the Vector Institute (\url{www.vectorinstitute.ai/partners}).

%% file: main/appendices/appendix.tex
\appendix
\input{main/appendices/notation.tex}
\input{main/appendices/proofs.tex}
\input{main/appendices/multivariate_lower.tex}
\input{main/appendices/hblemmas}
\input{main/appendices/multivariate_upper}
\input{main/appendices/experiment_details}

%% file: main/appendices/notation.tex
\section{Notation}
\label{app:notation}

\begin{table}[h]
    \centering
    \noindent\setlength\tabcolsep{4pt}
    \begin{tabular}{c|l}
         &  \textbf{Description} \\\hline
        $\spaceX$ & The domain of the data, e.g. $\bbR^d$ \\
        $\spaceY$ & The range of the data, e.g. $\bbR$ \\
        $\spaceZ$ & The product space $\spaceX \times \spaceY$ \\
        $\calP$ & A collection of distributions over $\spaceZ$ \\
        $\calJ$ & A (finite) subset of distributions in $\calP$ \\
        $P$ & An element of $\calP$\\
        $P^k$ & The product distribution, whose samples correspond to $k$ independent draws from $P$\\
        $P_{1:M}$ & The product distribution, $\Pi_{i=1}^{M} P_i$, for $P_i \in \calP$\\
        $\bar{P}_{1:M}$ & The mixture distribution, $\frac{1}{M}\sum_{i=1}^M P_i$, for $P_i \in \calP$\\
        $\Omega$ & A metric space, containing parameters for each distribution\\
        $\theta$ & A functional, mapping distributions in $\calP$ to parameters in $\Omega$ \\
        $\estimator$ & An estimator $\estimator: \spaceZ^n \rightarrow \calF$\\ 
        $S, \testS$ & Denotes training datasets drawn \iid from some $P \in \calP$.\\&Typically $S = \{z_1,\ldots,z_n\}$,$\testS = \{z'_1,\ldots,z'_k\}$\\
        $S_\calP$ & Denotes a meta-training set drawn \iid from $P_{1:M}$\\
        $[N]$, for $N \in \bbN$ & Indicates the set $\{1,\ldots,N\}$\\
        $\bball_p(r)$ & The $p$-norm ball of radius $r$, centered at $0$.
    \end{tabular}
    \caption{Summary of notation used in this manuscript}
    \label{tab:notation}
\end{table}

%% file: main/appendices/proofs.tex
\section{Lower Bound Proofs}
\label{app:proofs}

We will make use of several standard results below, which we present here.

\begin{lemma}\textbf{Fano's Inequality \citep{fano1961transmission, cover2012elements} }\label{lemma:fano}
For any estimator $\hat{Y}$ of a random variable $Y$ such that $Y \rightarrow Z \rightarrow \hat{Y}$ forms a Markov chain, it holds that,
\[
\bbP(\hat{Y} \neq Y) \geq \dfrac{H(Y|Z) - 1}{\log_2 |Y|} = \dfrac{H(Y) - I(Y;Z) - 1}{\log_2 |Y|}.
\]
\end{lemma}

\begin{lemma}\textbf{Mutual information equality \citep{khas1979lower}}\label{lemma:mi_ineq}
Consider random variables $Z_1, Z_2, Y$, then,
\[I(Y; (Z_1,Z_2)) + I(Z_1;Z_2) = I(Z_1; (Z_2, Y)) + I(Z_2; Y)\]
\end{lemma}

\begin{lemma}\textbf{Local packing lemma \citep{loh2017lower}}\label{lemma:local_packing}
Consider distributions $P_1,\ldots,P_J \in \calP$. Let $Y$ be a random variable distributed uniformly on $[J]$ and let $Z|\{Y=j\}$ be a vector of $k$ \iid samples from $P_{j}$. Then,
\[
I(Y;Z) \leq \frac{k}{J^2} \sum_{1 \leq i,j \leq J} \KL{P_i}{P_j}.
\]
\end{lemma}
We will require a novel local packing bound for the novel-task risk, which we present in Lemma~\ref{lemma:env_local_mixture_packing}.

\subsection{IID Lower Bound}\label{app:proofs:lbounds}

We first prove the \iid result, which will serve as a guide for our novel lower bounds.

\iidlbound*

\begin{proof}
First, notice that,
\[
\sup_{P \in \calP}\iidexploss \geq \dfrac{1}{J}\sum_{i=1}^J \bbE_{S\sim P_i^k}\left[ \loss{}{S}{P_i} \right].
\]
Now define the decision rule,
\[f(S) = \argmin_{1\leq j \leq J}\metricfn(\estimator(S), \theta_{P_j})],\]
with ties broken arbitrarily. We proceed by bounding the expected loss. First, using Markov's inequality,
\begin{align*}
\bbE_{S\sim P_i^k}\left[ \loss{}{S}{P_i} \right] &\geq \monofn(\delta)\bbP_{S\sim P^k_i}\left[\monofn(\metricfn(\estimator(S), \theta_{P_i})) \geq \monofn(\delta) \right], \\
&= \monofn(\delta)\bbP_{S\sim P^k_i}\left[\metricfn(\estimator(S), \theta_{P_i}) \geq \delta \right].
\end{align*}
Next, consider the case $\metricfn(\estimator(S), \theta_{i})) < \delta$. Through the triangle inequality,
\begin{align*}
    \metricfn(\estimator(S), \theta_{P_j}) &\geq \metricfn(\theta_{P_i}, \theta_{P_j}) - \metricfn(\estimator(S), \theta_{P_i}) \\
    &\geq 2\delta - \delta > \metricfn(\estimator(S), \theta_{P_i})
\end{align*}
Thus, the probability that the distance is less than $\delta$ is at least as large as the probability that the estimator is correct.
\[
\monofn(\delta)\bbP_{S\sim P^k_i}\left[ \lossfn(\estimator(S), \theta_{P_i}) \geq \monofn(\delta) \right] \geq \monofn(\delta) \bbP(f(S) \neq i). 
\]
Now, using Fano's inequality with $Y = \pi_{M+1}$, and $\hat{Y} = f(S)$ (and the corresponding Markov chain\\$\pi_{M+1} \rightarrow S \rightarrow f(S)$), we have,
\[
\dfrac{1}{J}\sum_{i=1}^J \bbP(f(S) \neq i) \geq \dfrac{\log_2 J - I(\pi_{M+1}; Z) - 1}{\log_2 J}.
\]
Combining the above inequalities with the Local Packing Lemma gives the final result.
\end{proof}

\subsection{Proof of Theorem~\ref{thm:env_lower_bound}}

\envlbound*

\begin{proof}
As in the \iid case, we first bound the supremum from below with an average,
\[\sup_{P'_1,\ldots,P'_{M+1} \in \calP} \exploss{(P'_{1:M})}{(P'_i)}{\envS}{\testS}
\geq 
\frac{1}{J}\sum^{J}_{i=1} \frac{1}{{J-1 \choose M}}\sum_{\pi|\{\pi_{M+1} = i\}}
\bbE_{\substack{w\sim W|\pi\\z \sim Z|\pi}}\left[ \loss{w}{z}{i} \right],
\]
where the inner sum is over all length $M$ orderings, $\pi$ with $\pi_{M+1} = i$.

\noindent As before, we consider the following estimator,
\[f(W,Z) = \argmin_{1\leq j\leq J} \rho(\hat{\theta}_{W}(Z), \theta(P_j))\]
Using Markov's inequality, and then following the proof of Theorem~\ref{thm:lower_bound}, we have,
\begin{align*}
    \frac{1}{J}\sum^{J}_{i=1} \frac{1}{{J-1 \choose M}}\sum_{\pi|\{\pi_{M+1} = i\}}
\bbE_{\substack{w\sim W|\pi\\z \sim Z|\pi}}\left[ \loss{w}{z}{i} \right]
&\geq\frac{1}{J}\sum^{J}_{i=1} \frac{1}{{J-1 \choose M}}\sum_{\pi|\{\pi_{M+1} = i\}}\monofn(\delta)\bbP[f(W,Z) \neq i\vert \pi] \\
    &= \monofn(\delta)\bbP[f(W,Z) \neq \pi_{M+1}]
\end{align*}
with the use of Fano's inequality, we arrive at,
\begin{align*}
\monofn(\delta)\left(1 - \dfrac{I(\pi_{M+1}; (W,Z)) + 1}{\log_2 J}\right)
\end{align*}
Conditioned on $Y$, each element of $W$ and $Z$ are independent but they are not identically distributed. Thus, with the application of Lemma~\ref{lemma:mi_ineq},
\[ I(\pi_{M+1}; (W,Z)) \leq I(\pi_{M+1};Z) + I(\pi_{M+1};W)\]
The result follows by combining these inequalities.
\end{proof}

\paragraph{Remark}
In the above proof of Theorem~\ref{thm:env_lower_bound}, we did not need to make use of the form of the distribution of $W|{Y=i}$, only that the correct graph structure was observed. This grants us some flexibility, which we utilize later in Section~\ref{app:mixture_lbounds} to prove lower bounds for mixture distributions.

\noindent We now proceed with proofs of the following corollaries.

\envlboundasymp*

\begin{proof}
This result follows as an application of the data processing inequality. Notice that $\pi_{M+1} \rightarrow \pi_{1:M} \rightarrow W$ forms a Markov chain. Thus,
\[I(\pi_{M+1};W) \leq I(\pi_{M+1};\pi_{1:M}),\]
by the data processing inequality. We can compute $I(\pi_{M+1};\pi_{1:M})$ in closed form:
\[I(\pi_{M+1};\pi_{1:M}) = \log \frac{J}{J-M}.\]
The proof is completed by plugging in the \iid local packing bound alongside the above.
\end{proof}

\subsection{Local packing results}\label{app:proofs:lpacking}

\begin{restatable}[Meta-learning local packing]{lemma}{looproductpacking}\label{lemma:env_local_product_packing}
Consider the same setting as in Theorem~\ref{thm:env_lower_bound}, then
\[I(\pi_{M+1}; W) \leq \frac{Mn}{J^2(J-1)} \sum_{1 \leq i,j \leq J} \KL{P_{i}}{P_{j}}\]
\end{restatable}

\begin{proof}
There are $(J-1)!/(J-M-1)!$ orderings on the first $M$ indices, given the $(M+1)^{th}$. We introduce the following notation,
\[\bar{P}_{-i} := \frac{(J-M-1)!}{(J-1)!} \sum_{\pi|\pi_{M+1}=i} p(W|\pi) \hspace{2cm} \bar{P} := \frac{1}{J} \sum_{i=1}^{J} \bar{P}_{-i}\]
As in previous proofs, we notice that we can write,
\[\bar{P} = \frac{1}{J} \bar{P}_{-i} + \frac{J-1}{J} \frac{1}{J-1} \sum_{j \neq i} \bar{P}_{-j}\]
First, note that we can upper bound $I(\pi_{M+1};W) \leq n I(\pi_{M+1};w)$, where $w$ denotes a single row in $W$. Further,
\begin{align*}
    I(\pi_{M+1}; w) &= \frac{1}{J}\sum_{i=1}^{J} \bbE \log \frac{\bar{P}_{-i}}{\bar{P}}\\
    &= \frac{1}{J}\sum_{i=1}^{J} \KL{\bar{P}_{-i}}{\frac{1}{J} \bar{P}_{-i} + \frac{J-1}{J} \frac{1}{J-1} \sum_{j \neq i} \bar{P}_{-j}} \\
    &\leq \frac{1}{J}\sum_{i=1}^{J} \frac{1}{J}\KL{\bar{P}_{-i}}{\bar{P}_{-i}} + \frac{J-1}{J} \KL{\bar{P}_{-i}}{\frac{1}{J-1} \sum_{j \neq i} \bar{P}_{-j}} \\
    &\leq \frac{1}{J(J-1)} \sum_{1 \leq i\neq j \leq J} \KL{\bar{P}_{-i}}{\bar{P}_{-j}}
\end{align*}
We will use the convexity of the KL divergence to upper bound this quantity. Each distribution $\bar{P}_{-i}$ is an average over a random selection of index orderings.

When applying convexity, all pairs of selections that exactly match will lead to a KL divergence of zero. There are the same number of these in each component of $\bar{P}_{-i}$. Thus we care only about selections that contain either $j$ or $i$ such that matching pairs of distributions exactly is not possible. Further, we need only consider pairs of product distributions who differ only in a single, identical position.

Each of the above described pairs of distributions has KL divergence equal to $\KL{P_{j}}{P_{i}}$. We conclude by counting the total number of orderings producing such pairs. First, there are $M$ choices for the index of $P_j$ and $P_i$. Then, there are $(J-2)!/(J-M-1)!$ total orderings of the remaining $M-1$ elements. Thus, we have,

\begin{align*}
I(\pi_{M+1};w) &\leq \frac{1}{J(J-1)} \sum_{1 \leq i\neq j \leq J} \KL{\bar{P}_{-i}}{\bar{P}_{-j}}\\
&\leq \frac{1}{J(J-1)} \frac{(J-M-1)!}{(J-1)!}\frac{M(J-2)!}{(J-M-1)!} \sum_{1 \leq i\neq j \leq J} \KL{P_{j}}{P_{i}}\\
&= \frac{M}{J(J-1)^2} \sum_{1 \leq i\neq j \leq J} \KL{P_{j}}{P_{i}}\\
\end{align*}
\end{proof}

\subsection{Bounds using mixture distributions}\label{app:mixture_lbounds}

In this section we introduce tools to lower bound the minimax risk when the meta-training set is sampled from a mixture over the meta-training tasks, $\bar{P}_{1:M} = \frac{1}{M}\sum_{i=1}^{M} P_i$. We note first that Theorem~\ref{thm:env_lower_bound} can be reproduced exactly when $W \sim \bar{P}_{1:M}$. Thus, we need only provide a local packing bound for the mixture distribution. In Lemma~\ref{lemma:env_local_mixture_packing} we provide such a lower bound for the special case where $M = J-1$, so that data is sampled from a mixture over the entire environment.

\begin{restatable}[Leave-one-task-out mixture local packing]{lemma}{loopacking}\label{lemma:env_local_mixture_packing}
Let $\calJ \subset \calP$ contain $J$ distinct distributions such that $\metricfn(\theta_{P}, \theta_{P'})~\geq~2\delta$ for all $P,P' \in \calJ$ and let $\bar{P}_{-i} = \frac{1}{J-1} \sum_{j \neq i} P_j$. Let $\pi$ be a random ordering of the $J$ elements, and define $W|\pi$ to be a vector of $n$ \iid samples from $\bar{P}_{-\pi_{M+1}}$. Then,
\[I(\pi_{M+1}; W) \le \frac{1}{(J-1)J^2} \sum_{1 \le i,j \le J} \KL{P_i}{P_j}.\]
\end{restatable}

\begin{proof}
From Lemma~\ref{lemma:local_packing} (and some simple arithmetic) we have,
\begin{align*}
I(\pi_{M+1}; W) &= \frac{1}{J} \sum_{i=1}^{J} \KL{\barP_{-i}}{\barP_{1:J}}.
\end{align*}
Note that by the definition of the mixture distribution,
\[
\barP_{1:J} = \frac{J-1}{J} \barP_{-i} + \frac{1}{J} P_i.
\]
Using the convexity of the KL divergence,
\begin{align*}
I(\pi_{M+1}; W) &= \frac{1}{J} \sum_{i=1}^{J} D_{\mathrm{KL}}\left(\barP_{-i}\Big\Vert\frac{J-1}{J} \barP_{-i}
+ \frac{1}{J} P_i\right)\\ 
&\le \frac{1}{J} \sum_{i=1}^{J} \frac{J-1}{J} \KL{\barP_{-i}}{\barP_{-i}} + \frac{1}{J} \KL{\barP_{-i}}{P_i}\\
&=\frac{1}{J^2} \sum_{i=1}^{J} \KL{\barP_{-i}}{P_i}\\
&= \frac{1}{J^2} \sum_{i=1}^{J} D_{\mathrm{KL}}\left(\frac{1}{J-1}\sum_{1 \le j \le J, j \neq i} P_j \Big\Vert P_i \right)\\
&\le \frac{1}{(J-1)J^2} \sum_{i=1}^{J} \sum_{1 \le j \le J, j \neq i} \KL{P_j}{P_i}\\
&= \frac{1}{(J-1)J^2} \sum_{1 \le i,j \le J} \KL{P_i}{P_j}.
\end{align*}
Noting for the last step that the KL is zero if and only if the distributions are the same almost everywhere.
\end{proof}

%% file: main/appendices/multivariate_lower.tex
\subsection{Proof of Hierarchical linear model lower bound}
\label{app:proofs:linear_lower}

Recall that the space of distributions we consider is given by,
\[\calP_{LR} = \{ p_{\lregparams}(\by) = \calN(X\lregparams, \sigma^2 I): \lregparams \in \bball_2(1), X \in \bbR^{n \times d}. \} \]

\mlreglbound*

\begin{proof}The proof consists of two steps, we first construct a $2\delta$-packing of $\calP_{LR}$. Then, we upper bound the KL divergence between two distributions in this packing and use Corollary~\ref{corollary:env_lbound_packing} to give the desired bound.

\noindent The maximal packing number $J$ for the unit 2-norm ball can be bounded by the following,
\[ \left(\frac{1}{\delta}\right)^d\leq J \leq \left(1 + \frac{2}{\delta}\right)^d.\]
We use a common scaling trick. First, through this bound, we can build a packing set, $\calV$, with packing radius $1/2$, giving $2^d \leq J \leq 5^d$. We define a new packing set of the same cardinality by taking $\theta_i = 4\delta v_i$ for all $v_i \in \calV$. Giving for all $i \neq j$,
\[ \Vert\theta_i - \theta_j\Vert = 4\delta\Vert v_i - v_j\Vert \geq 2\delta\]
similarly, $\Vert \theta_i - \theta_j\Vert \leq 4\delta$.

\noindent We now proceed with bounding the KL divergences.
\begin{align*}
    \KL{P_i}{P_j} &= \frac{1}{2\sigma^2}\Vert X_i\lregparams_i - X_j\lregparams_j\Vert_2^2\\
    &= \frac{1}{2\sigma^2}\left(\lregparams_i^\top X_i^\top X_i \lregparams_i + \lregparams_j^\top X_j^\top X_j \lregparams_j - 2\lregparams_i^\top X_i^\top X_j \lregparams_j \right)\\
    &\leq \frac{1}{2\sigma^2}\left(n_i \gamma_i^2 \Vert \lregparams_i \Vert^2 + n_j \gamma_j^2 \Vert \lregparams_j \Vert^2 - 2\lregparams_i^\top X_i^\top X_j \lregparams_j \right)
\end{align*}
where $\gamma_i^2 = \sup_{\lregparams} \frac{\Vert X_i \lregparams \Vert}{\sqrt{n_i}\Vert \lregparams \Vert}$. We write $n = \max_k{n_k}$ and $\gamma = \max_{k}{\gamma_k}$, then,
\begin{align*}
    \KL{P_i}{P_j} &\leq \frac{n \gamma^2}{2\sigma^2}\left(\Vert \lregparams_i \Vert^2 + \Vert \lregparams_j \Vert^2 - \frac{2}{n\gamma^2}\lregparams_i^\top X_i^\top X_j \lregparams_j \right)\\
    &\leq \frac{n \gamma^2}{2\sigma^2}\left(\Vert \lregparams_i \Vert^2 + \Vert \lregparams_j \Vert^2 + 2\Vert\lregparams_i\Vert \Vert\lregparams_j\Vert \right)\\
    &= \frac{n \gamma^2}{2\sigma^2}\left( \Vert\lregparams_i\Vert + \Vert\lregparams_j\Vert\right)^2 \leq \frac{32 n \gamma^2 \delta^2}{\sigma^2}
\end{align*}
The second line is derived using the Cauchy-Schwarz inequality, and the final inequality uses $\Vert \lregparams_i \Vert = \Vert 4\delta v_i \Vert \leq 4\delta$.

\noindent Now, using Corollary~\ref{corollary:env_lbound_packing},
\[R^*_{\calP_{LR}} \geq \delta^2\left(1 - \frac{(nM 2^{-d} + k)32\gamma^2\delta^2 / \sigma^2 + 1}{d}\right),\]
Choosing $\delta^2 = d\sigma^2 / 64\gamma^2 (2^{-d} Mn + k)$ gives,
\[1 - \frac{(nM 2^{-d} + k)32\delta^2 / \sigma^2 + 1}{d} = 1 - \frac{d/2 + 1}{d} \geq 1/4,\]
for $d \geq 2$. Thus,
\[R^*_\calP \geq O\left(\frac{d\sigma^2}{\gamma^2 (2^{-d}nM + k)}\right)\]
\end{proof}

%% file: main/appendices/hblemmas.tex

\section{Hierarchical Bayesian Linear Regression Upper Bounds}
\subsection{Some useful linear algebra results}
Let $\smax(A)$ denotes the maximum singular value of $A$; $\smin(A)$ denotes the minimum singular value of $A$.

\begin{lemma}\textbf{Singular value of sum of two matrices}\label{lemma:sigma_sum_lemma}
Let $A, B \in \bbR^{m \times n}$, then
$\smax(A) + \smax(B) \ge \smax(A+B)$. Furthermore, if $A, B$ are positive definite, 
$\smin(A) + \smin(B) \le \smin(A+B)$.
\end{lemma}
\begin{proof}
The first result follows immediately from the triangle inequality of the matrix norm $\norm{\cdot}_2$.

\noindent For the second result, suppose that $A$ and $B$ are positive definite.
\begin{align}
\smin(A + B) &= \inf_{\norm{u}=1} \norm{(A+B)u} \\
&= \sqrt{\inf_{\norm{u}=1} \norm{(A+B)u}^2} \\
&= \sqrt{\inf_{\norm{u}=1} \norm{Au}^2 + \norm{Bu}^2 + 2\langle Au, Bu \rangle } \\
&= \sqrt{\inf_{\norm{u}=1} \norm{Au}^2 + \norm{Bu}^2 + 2u^\top A^\top Bu } \\
\end{align}
Now, notice that $A^\top B$ is similar to the matrix $A^{1/2}BA^{1/2}$, which exists as $A$ is positive definite. This matrix is itself positive definite, and thus has non-negative eigenvalues, meaning $A^\top B$ also has all positive eigenvalues. Thus, $u^\top A^\top Bu \ge 0$, for all $u$, and,
\begin{align}
\smin(A + B) & \ge \sqrt{\inf_{\norm{u}=1} \norm{Au}^2 + \norm{Bu}^2} & \\
& \ge \sqrt{\inf_{\norm{u}=1} \norm{Au}^2 + \inf_{\norm{v}=1} \norm{Bv}^2} \\
& \ge \sqrt{\smin^2(A) + \smin^2(B)} \\
& \ge \smin(A) + \smin(B) & (\text{Concavity of} \ \ \sqrt{\cdot})
\end{align}
\end{proof}

\begin{lemma}
\textbf{Singular value of product of two matrices}\label{lemma:sigma_prod_lemma}
Let $A, B \in \mathbb{C}^{n \times n}$, then
$\smax(A)\smax(B) \ge \smax(AB)$, and, $\smin(A)\smin(B) \le \smin(AB)$.
\end{lemma}
First we prove the maximum singular value.
\begin{proof}
\begin{align}
\smax(AB) &= \sup_{\norm{v} = 1} \sqrt{v^* B^* A^* A B v} \\
&= \sup_{\norm{v} = 1} \sqrt{\norm{Bv}^2 u^* A^* A u} \ \ \text{for } u = \frac{Bv}{\norm{Bv}}, \\
&\le \sup_{\norm{v} = 1, \norm{u} = 1} \sqrt{\norm{Bv}^2 u^* A^* A u}\\
&= \sqrt{\sup_{\norm{v} = 1}  \norm{Bv}^2 \sup_{\norm{u} = 1} \norm{Au}^2}\\
&= \sqrt{\smax^2(B) \smax^2(A)}\\
&= \smax(A)\smax(B).
\end{align}
The minimum singular value follows a similar structure. Suppose $AB$ is full rank,
\begin{align}
\smin(AB) &= \inf_{\norm{v} = 1} \sqrt{v^* B^* A^* A B v} \\
&= \inf_{\norm{v} = 1} \sqrt{\norm{Bv}^2 u^* A^* A u} \ \ \text{for } u = \frac{Bv}{\norm{Bv}}, \\
&\ge \inf_{\norm{v} = 1, \norm{u} = 1} \sqrt{\norm{Bv}^2 u^* A^* A u} \\
&= \sqrt{\inf_{\norm{v} = 1}  \norm{Bv}^2 \inf_{\norm{u} = 1} \norm{Au}^2}\\
&= \sqrt{\smin^2(B) \smin^2(A)}\\
&= \smin(A)\smin(B).
\end{align}
If $AB$ is not full rank, then $\smin(AB) = \smin(A) \smin(B) = 0$.
\end{proof}

\begin{lemma}(\textbf{Von Neumann's Trace Inequality} \citep{von1937some})\label{lemma:von_neumann_trace}
Given two $n\times n$ complex matrices $A,B$, with singular vales $a_1 \geq\ldots\geq a_n$ and $b_1\geq\ldots\geq b_n$ respectively. We have,
\[\left\vert \tr(AB)\right\vert \leq \sum_{i=1}^{n} a_i b_i\]
\end{lemma}
This is a classic result whose proof we exclude.

\noindent As a direct consequence of Lemma~\ref{lemma:von_neumann_trace}, $\left\vert \tr(AB)\right\vert \leq n a_1 b_1$.

%% file: main/appendices/multivariate_upper.tex
\subsection{Posterior estimate}
\label{app:proofs:ubounds}

For reference, we reproduce the posterior estimate for the true parameters $\thetaY$.  As a shorthand, we write $\allData = (\by_1,\ldots, \by_{M+1})$.
\begin{align}
p(\btau \vert \sourceData) &= \mathcal{N}(\mu_{\btau \vert \sourceData}, \Sigma_{\btau \vert \sourceData}),\\
\CovPosXTau^{-1} &= \PrecPosXTauExpanded, \\
\muPosXTau &= \muPosXTauExpanded
\end{align}
\begin{align}
p(\thetaY \vert \sourceData, \novelData) &= \mathcal{N}(\muPosTheta, \CovPosTheta),\\
\CovPosXThetaTau &= \CTheta + \CovPosXTau \\
\CovPosTheta^{-1} &= \XTCYX + \CovPosXThetaTau^{-1}\\
\muPosTheta &=
\CovPosTheta (\XTCY\novelData + \CovPosXThetaTau^{-1} \muPosXTau)
\end{align}

\subsection{Upper bound for meta linear regression}

In this section we prove the main upper bound result of our paper, Theorem~\ref{thm:lregression_bias_variance}.
\mlrbv*

\noindent Before proceeding with the proof, we introduce some additional notation and technical results.

\paragraph{Additional notation} To alleviate (only a little of) the notational clutter, we will define the following quantities,
\begin{itemize}
\item $\CPos = \CovPosTheta$
\item $\CPosT = \CovPosXThetaTau$.
\item $\smin(\DX / \sqrt{\nX}) = \sminX$
\item $\smin(\DY / \sqrt{\nY}) = \sminY$
\item $\smax(\DX / \sqrt{\nX}) = \gamma_1 = \condX \sminX$
\item $\smax(\DY / \sqrt{\nY}) = \gamma_2 = \condY \sminY$
\item $\alphaX = \smaxCX / \smaxP$
\item $\alphaY = \smaxCY / \smaxP$
\item $\LConst = \frac{\alphaY}{(M + \condX^2)\sminY^2}$
\item $\LLConst = \frac{\alphaX}{\sminX^2\sminY^2 \condY^2}$
\item $\LLLConst = \frac{\condtC \condtau \alphaY}{2\sminY^2 \condY^2}$
\item $\MConst = \frac{\sminY^2 \alphaX}{\sminX^2 \alphaY}$
\item $\MMConst = \sminY^2 \condY^2$
\item $\MMMConst = \frac{\alphaX \sminY^2 \condY^2}{\condtau^2 \sminX^2 \alphaY}$
\end{itemize}

As we have uniform bounds on the singular values of all design matrices, we introduced an auxillary matrix $X$ whose largest and smallest singular values are given by $\sqrt{n}\gamma_1$ and $\sqrt{n}s_1$ respectively.

\noindent We will also write $S(A) = \Covop[A, A]$, and $\cond(A) = \smax(A) / \smin(a)$ throughout.

\paragraph{Bias-Variance Decomposition}

As is standard, we can decompose the risk into the bias and variance of the estimator:
\begin{align}
\bbE[\norm{ \thetaEst - \thetaY }^2]  
&= \bbE[\tr( (\thetaEst - \thetaY) (\thetaEst - \thetaY)^\top)] \\
&= \tr( \bbE [(\thetaEst - \thetaY) (\thetaEst - \thetaY)^\top] ) \\
&= \tr( \Covop [\thetaEst, \thetaEst]) + 
\tr( \bbE[(\thetaEst - \thetaY)] \bbE[(\thetaEst - \thetaY)]^\top )
\end{align}

\noindent In the next two sections, we will derive upper bounds on the bias and variance terms above.

\subsubsection{Variance technical lemmas}

We first decompose the variance into contributions from two sources: the variance from data in the novel task and the variance from data in the source tasks.

\begin{lemma}(\textbf{Variance decomposition})\label{lemma:variance_decomp}
Let $\thetaEst = \muPosTheta$ as defined above. Then the variance of the estimator can be written as
\[\tr(S(\thetaEst)) = \tr(\CPos \XTCYX \CPos) + \tr(S(\CPos {\CPosT}^{-1} \muPosXTau)) \]
\end{lemma}
\begin{proof}
\begin{align}
\tr(\Covop [\thetaEst, \thetaEst]) &= \tr(S(\CPos (\XTCY \novelData + {\CPosT}^{-1} \muPosXTau)))\\
&=\tr(S(\CPos \XTCY \novelData) + S(\CPos {\CPosT}^{-1}  \muPosXTau))\\
&=\tr(\CPos \XTCY S(\novelData) \CYInv \DY \CPos + S(\CPos {\CPosT}^{-1} \muPosXTau))\\
&=\tr(\CPos \XTCYX \CPos) + \tr(S(\CPos {\CPosT}^{-1} \muPosXTau))\label{eqn:variance_decomp}
\end{align}
\end{proof}

We will now work towards a bound for each of the two variance terms in Lemma~\ref{lemma:variance_decomp} separately. To do so, we will need to produce bounds on the singular values of terms appearing in Lemma~\ref{lemma:variance_decomp}.

\noindent We begin with the covariance term $\CPos$.

\begin{lemma}(\textbf{Novel task covariance singular value bound})\label{lemma:novel_tasks_cov_bound}
Let $\LConst$, $\MConst$ and $\sminY$ be as defined above. Then,
\[\smax(\CPos) \leq \frac{\smaxCY}{\sminY^2}
\left[ \nY + 
\frac{\nX}{\frac{M\nX}{\LConst} + \MConst} \right]^{-1}. \]
\end{lemma}

\begin{proof}
Using Lemma~\ref{lemma:sigma_sum_lemma}, we can bound $\smax(\CPos)$ as follows,
\begin{align}
\smax(\CPos) &= \smax(\CovPosTheta) \\
&= \smax((\XTCYX + \CovPosXThetaTau^{-1})^{-1}) \\
&= 1/\smin(\XTCYX + \CovPosXThetaTau^{-1}) \\
&\leq 1/(\smin(\XTCYX) + \smin(\CovPosXThetaTau^{-1})) \\
&= 1/(\smin(\XTCYX) + 1/\smax(\CovPosXThetaTau))
\end{align}
Now, using the auxillary matrix $X$,
\begin{align}
\smax(\CPos) &\le \left[ \smin(\XTCYX) + \frac{1}{\smaxP + \frac{1}{M}\smax((\DX^\top \posC{}^{-1} \DX)^{-1})}\right]^{-1}\\
&= \left[ \frac{\smin(\DY^\top \DY)}{\smaxCY} + \frac{1}{\smaxP + \frac{1}{M}\frac{1}{\smin(\DX^\top \posC{}^{-1} \DX)}} \right]^{-1}\\
&\le \left[ \frac{\nY\smin^2(\DY/\sqrt{\nY})}{\smaxCY} + \frac{1}{\smaxP + \frac{1}{M}\frac{1}{\smin(\DX^\top \posC{}^{-1} \DX)}} \right]^{-1} \\
&\le \left[\frac{\nY\sminY^2}{\smaxCY} + \frac{1}{\smaxP + \frac{1}{M}\frac{1}{\smin(\DX^\top (\DXP\DX^\top + \Ci{})^{-1} \DX)}} \right]^{-1}
\end{align}

Above we have used Lemma~\ref{lemma:sigma_sum_lemma} repeatedly, alongside the standard identity, $\smax(A^{-1}) = \smin(A)^{-1}$. We continue now, additionally using Lemma~\ref{lemma:sigma_prod_lemma},
\begin{align}
\smax(\CPos) &\le \left[ \frac{\nY\sminY^2}{\smaxCY} + \frac{1}{\smaxP + \frac{1}{M}\frac{1}{\smin(\DX^\top \DX) \smin((\DXP\DX^\top + \CX)^{-1})} } \right]^{-1} \\
&= \left[ \frac{\nY\sminY^2}{\smaxCY} + \frac{1}{\smaxP + \frac{\smax(\DXP\DX^\top + \CX)}{\smin(\DX^\top \DX)}} \right]^{-1} \\
&\le \left[ \frac{\nY\sminY^2}{\smaxCY} + \frac{1}{\smaxP + \frac{1}{M}\frac{\smax(\DXP\DX^\top) +\smaxCX}{\nX \sminX^2}} \right]^{-1} \\
&\le \left[ \frac{\nY\sminY^2}{\smaxCY} + \frac{1}{\smaxP +\frac{\smaxP\smax(\DX\DX^\top) + \smaxCX}{\nX \sminX^2}} \right]^{-1} \\
&= \left[ \frac{\nY\sminY^2}{\smaxCY} + \frac{1}{\smaxP + \frac{1}{M}\frac{\smaxP\nX \sminX^2 \condX^2 + \smaxCX}{\nX \sminX^2}} \right]^{-1} \\
&=\smaxCY
\left[ \nY\sminY^2 + 
\frac{M\nX}{\frac{\nX (M +\condX^2)}{\alphaY} + \frac{\alphaX}{\sminX^2 \alphaY}} \right]^{-1} \\
&=\frac{\smaxCY}{\sminY^2}
\left[ \nY + 
\frac{\nX}{\frac{M\nX}{\LConst} + \MConst} \right]^{-1}.
\end{align}
\end{proof}

Next, we deal with terms appearing corresponding to the data from the source tasks.

\begin{lemma}(\textbf{Source tasks covariance singular value bound})\label{lemma:source_tasks_cov_bound}
Let $\posCX = \XPXT + \CX$, and write $\condtC = \cond(\posCX)$ and $\condtau = \cond(\CovPosXTau)$. Then,
\[\smax^2(\CovPosXThetaTau^{-1} \CovPosXTau)
\le \frac{1}{\frac{2M\sminP\nX \sminX^2}{\smax(\posCX) \condtau} + \frac{1}{\condtau^2}} =: D_1 \]
and,
\[\smax(\CXs \posCX^{-1}) \le \frac{1}{ \frac{\nX\sminX^2}{\alphaX} + 1} =: D_2 \]
\end{lemma}

\begin{proof}
Using Lemma~\ref{lemma:sigma_sum_lemma} and Lemma~\ref{lemma:sigma_prod_lemma} we have,

\begin{align}
\smax^2(\CovPosXThetaTau^{-1} \CovPosXTau)
&= \smax^2(\CovPosXThetaTau^{-1}\CovPosXTau) \\
&\le \smax^2(\CovPosXThetaTau^{-1}) \smax^2(\CovPosXTau) \\
&= \smin^{-2}(\CovPosXThetaTau) \smax^2(\CovPosXTau) \\
&= \smin^{-2}(\CTheta + \CovPosXTau) \smax^2(\CovPosXTau) \\
&\le \frac{\smax(\CovPosXTau)^2}{(\sminP + \smin(\CovPosXTau))^2}
\end{align}
Now, using $\sminP > 0$,
\begin{align}
\frac{\smax(\CovPosXTau)^2}{(\sminP + \smin(\CovPosXTau))^2} &\le \frac{\smax(\CovPosXTau)^2}{2\sminP\smin(\CovPosXTau) + \smin(\CovPosXTau)^2} \\
&\le \frac{1}{\frac{2\sminP}{\smax(\CovPosXTau) \condtau} + \frac{1}{\condtau^2}}
\end{align}
Introducing the auxillary matrix $X$ and using Lemma~\ref{lemma:sigma_sum_lemma} and Lemma~\ref{lemma:sigma_prod_lemma} on $\CovPosXTau$, we have
\begin{align}
\frac{1}{\frac{2\sminP}{\smax(\CovPosXTau) \condtau} + \frac{1}{\condtau^2}} \le \frac{1}{\frac{2M\sminP\smin(\DX^\top \posC{}^{-1} \DX)}{\condtau} + \frac{1}{\condtau^2}},
\end{align}
where,
\begin{align}
\smin(\DX^\top \posCX^{-1} \DX) &\geq \frac{\smin(\DX^\top \DX)}{\smax(\posCX)}\\
&= \frac{\nX \sminX^2}{\smax(\posCX)}.
\end{align}
This gives the first stated inequality,
\begin{align}
\smax^2(\CovPosXThetaTau^{-1} \CovPosXTau)
&\le \frac{1}{\frac{2M\sminP\nX \sminX^2}{\smax(\posCX) \condtau} + \frac{1}{\condtau^2}} =: D_1
\end{align}
The second follows as,
\begin{align}
\smax(\CXs \posCX^{-1}) &= \frac{\smaxCX}{\smin(\XPXT + \CX)} \\
&\le\frac{\smaxCX}{\smin(\XPXT) + \smin(\CX)}\\
&\le\frac{\smaxCX}{\nX\sminX^2 \sminP + \sminCX}\\
&=\frac{1}{ \frac{\nX\sminX^2}{\alphaX} + 1} =: D_2
\end{align}
\end{proof}

\noindent In Lemma~\ref{lemma:source_tasks_cov_bound}, we introduced additional condition numbers, which we can bound as follows,
\begin{align}
\condtC &= \cond(\posCX) = \cond(\XPXT + \CX) 
\le \cond(\XPXT) \le \cond(\XPXT) \cond(\CTheta) = \condX^2,\\
\condtau &= \cond(\CovPosXTau) \le \cond(\DX^\top \DX) \cond(\posCX) = \condX^2 \condtC \le \condX^4.
\end{align}

\subsubsection{Variance upper bound}
We are now ready to put the above technical results together to achieve a bound on the variance of the estimator. 

\begin{lemma}(\textbf{Variance bound})\label{lemma:variance_ubound}
\[ \tr(S(\thetaEst)) \le \frac{\condY^2 \sCY^2}{\sminY^2} \dim \left[ \nY + \frac{\nX}{\frac{\nX}{\LConst} + \MConst} \right]^{-2} 
\left[ \nY + \frac{M\nX}{ 
(\frac{\nX}{\LLConst} + \MMConst ) 
(\frac{M\nX}{\LLLConst} + \MMMConst)} \right] \]
\end{lemma}

\begin{proof}
First, by Lemma~\ref{lemma:variance_decomp} we can decompose the overall variance into two terms:
\[\tr(S(\thetaEst)) = \tr(\CPos \XTCYX \CPos) + \tr(S(\CPos {\CPosT}^{-1} \muPosXTau)) \]
We deal with the left term first.

Using trace permutation invariance and the von Neumann trace inequality (Lemma~\ref{lemma:von_neumann_trace}). We can upper bound the left variance term as follows,
\begin{align}
\tr(\CPos \XTCYX \CPos) &=  \sminCYInv \tr(\CPos\CPos \DY^\top \DY)\\
&\le \dim \nY \sminCYInv \smax(\CPos)^2 \smax^2(\DY/\sqrt{\nY})\\
&= \dim \nY \sminCYInv \smax(\CPos)^2 \sminY^2 \condY^2 
\end{align}
For the second variance term, we observe that,
\begin{align}
& \ \ \ \ \ \tr(\CPos \CPosT^{-1} S(\muPosXTau) \CPosT^{-1} \CPos) \\
&=\tr(\CPos \CPosT^{-1} S(\muPosXTauExpanded) \CPosT^{-1} \CPos)\\
&\leq M\tr(\CPos \CPosT^{-1} \CovPosXTau \DX^\top \posC{}^{-1} S(y_1) \posC{}^{-1} \DX \CovPosXTau \CPosT^{-1} \CPos)\\
&=M\tr(\CPos \CPosT^{-1} \CovPosXTau \DX^\top \posCX^{-1} \CX \posCX^{-1} \DX \CovPosXTau \CPosT^{-1} \CPos)\\
&\le M\smax(\CPos)^2 \smax^2(\CovPosXThetaTau^{-1} \CovPosXTau) \smaxCX \tr(\DX^\top \posC{}^{-1} \posC{}^{-1} \DX)
\end{align}
Using Lemma~\ref{lemma:source_tasks_cov_bound}, we have,
\begin{align}
\tr(\CPos \CPosT^{-1} S(\muPosXTau) \CPosT^{-1} \CPos) &\le \smax(\CPos)^2 M D_1 D_2 \tr(\DX^\top \DX) \smax(\posCX^{-1})\\
& \le \smax(\CPos)^2 M D_1 D_2 \min(\nX,\dim) \nX \smax(\posCX^{-1})\\
& \le \smax(\CPos)^2 D_2 \frac{M\min(\nX,\dim) \nX}{\frac{2M\sminP\nX \sminX^2 \smin(\posCX)}{\smax(\posCX)\condtau} + \frac{\smin(\posCX)}{\condtau^2}} \ \\
& \le \smax(\CPos)^2 D_2 \frac{M\min(\nX,\dim) \nX}{\frac{2M\sminP\nX \sminX^2 \smin(\posCX)}{\smax(\posCX)\condtau} + \frac{\sminCX}{\condtau^2}} \ \\
& \le \smax(\CPos)^2 D_2 \frac{\min(\nX,\dim) \nX}{\sCY^2} \frac{M}{\frac{2M\nX}{\condtC \condtau \alphaY} + \frac{\alphaX}{\condtau^2 \sminCX^2 \alphaY}} \ \\
& \le  \frac{\smax(\CPos)^2}{\sCY^2}  \frac{\nX \dim}{ \frac{\nX\sminX^2}{\alphaX} + 1}  \frac{M}{\frac{2M \nX}{\condtC \condtau \alphaY} + \frac{\alphaX}{\condtau^2 \sminX^2 \alphaY}}
\end{align}
Finally, rearranging and using Lemma~\ref{lemma:novel_tasks_cov_bound}, we can bound the sum of the two terms in the variance as follows,
\begin{align}
\tr(S(\thetaEst))
& \le \frac{\smax(\CPos)^2}{\sCY^2}
(\nY \dim \sminY^2 \condY^2  + 
\frac{\nX \dim}{ \frac{\nX\sminX^2}{\alphaX} + 1} \frac{M}{\frac{2M\nX}{\condtC \condtau \alphaY} + \frac{\alphaX}{\condtau^2 \sminX^2 \alphaY}})\\
& \le \frac{\smax(\CPos)^2 \sminY^2 \condY^2}{\sCY^2}
(\nY \dim + 
\frac{\nX \dim}{ \frac{\nX\sminX^2\sminY^2 \condY^2}{\alphaX} + \sminY^2 \condY^2} \frac{M}{\frac{2M \nX \sminY^2 \condY^2}{\condtC \condtau \alphaY} + \frac{\alphaX \sminY^2 \condY^2}{\condtau^2 \sminX^2 \alphaY}})\\
& \le \frac{\condY^2 \sCY^2}{\sminY^2} \left[ \nY + \frac{\nX}{\frac{\nX}{\LConst} + \MConst} \right]^{-2} 
\left[ \nY \dim + \frac{M\nX \dim}{ 
(\frac{\nX\sminX^2\sminY^2 \condY^2}{\alphaX} + \sminY^2 \condY^2 ) 
(\frac{2M \nX \sminY^2 \condY^2}{\condtC \condtau \alphaY} + \frac{\alphaX \sminY^2 \condY^2}{\condtau^2 \sminX^2 \alphaY})} \right]\\
& \le \frac{\condY^2 \sCY^2}{\sminY^2} \dim \left[ \nY + \frac{\nX}{\frac{\nX}{\LConst} + \MConst} \right]^{-2} 
\left[ \nY + \frac{M\nX}{ 
(\frac{\nX}{\LLConst} + \MMConst ) 
(\frac{M\nX}{\LLLConst} + \MMMConst)} \right]
\end{align}
\end{proof}

\subsubsection{Bounding the Bias}
\begin{lemma}(\textbf{Bias upper bound})\label{lemma:bias_ubound}
Given $\theta_1,\ldots,\thetaY \in \bball_2(1)$, we have,
\[\bbE[(\thetaEst - \thetaY)] \le O\left(\dim \left[ \nY + 
\frac{M\nX}{\frac{\nX (M +\condX^2)\sminY^2}{\alphaY} + A} \right]^{-2}\right) \]
\end{lemma}
\begin{proof}
The bias can be computed as follows,
\begin{align}
\bbE[(\thetaEst - \thetaY)] &= \bbE \muPosTheta - \thetaY\\
&= \CovPosTheta \bbE (\XTCY y_2 + \CovPosXThetaTau^{-1} \muPosXTau) - \thetaY \\
&= \CovPosTheta (\XTCYX \thetaY + \CovPosXThetaTau^{-1} \bbE\muPosXTau) -\thetaY\\
&= (\XTCYX + \CovPosXThetaTau^{-1})^{-1}(\XTCYX \thetaY + \CovPosXThetaTau^{-1} \bbE\muPosXTau) - \thetaY \\
&= (F+G)^{-1} (F \thetaY + G \muPosXTau) - \thetaY\\
&= (F+G)^{-1} F \thetaY - (F+G)^{-1} (F + G) \thetaY + (F+G)^{-1} G \bbE\muPosXTau\\
&= (F+G)^{-1} G (\bbE\muPosXTau- \thetaY),
\end{align}
where we wrote $F = \XTCYX$, and $G = \CovPosXThetaTau^{-1}$. Thus,
\[\bbE[(\thetaEst - \thetaY)]^\top\bbE[(\thetaEst - \thetaY)] \le \norm{(F+G)^{-1}}_2^2\norm{G}_2^2 \norm{\bbE\muPosXTau- \thetaY}_2^2\]
We can bound each term in turn. First, note that $\norm{G}_2^2 \le 1/\sminP$, and we have bounded $\norm{(F+G)^{-1}}_2^2$ above. We can write,
\begin{align*}
    \norm{\bbE\muPosXTau- \thetaY}_2^2 &= \bignorm{\CovPosXTau\left(\sum^{M}_{i=1}\DXi{i}^\top \posC{i}^{-1}\DXi{i}\theta_i \right) - \thetaY}_2^2\\
    \norm{\bbE\muPosXTau- \thetaY}_2^2 &= \bignorm{\CovPosXTau\left(\sum^{M}_{i=1}\DXi{i}^\top \posC{i}^{-1}\DXi{i}\theta_i \right) - \CovPosXTau\CovPosXTau^{-1}\thetaY}_2^2\\
    &= \bignorm{\CovPosXTau \sum^{M}_{i=1}\DXi{i}^\top\posC{i}^{-1}\DXi{i}(\theta_i - \thetaY)}_2^2\\
    &\le \left(\sum^{M}_{i=1}\norm{\CovPosXTau}_2 \ \norm{\DXi{i}^\top\posC{i}^{-1}\DXi{i}(\theta_i - \thetaY)}_2\right)^2\\
    &\le \left(\sum^{M}_{i=1}\norm{\CovPosXTau}_2 \ \norm{\DXi{i}^\top\posC{i}^{-1}\DXi{i}}_2\right)^2
\end{align*}
The last line follows from the fact that the parameters lie in a ball of unit radius. We now proceed by bounding the sum by $M$ times the supremum --- with some light abuse of notation,
\begin{align*}
    \norm{\muPosXTau- \thetaY}_2^2 &\le (\smax(X^\top \posC{}^{-1} X)\smax(\DXi{}\posC{}^{-1}\DXi{}))^2\\
    &= \smax(X^\top \posC{}^{-1} X)^4 \le O(1)\\
\end{align*}
Thus, overall the convergence of the bias is bounded by,
\[\bbE[(\thetaEst - \thetaY)]^\top\bbE[(\thetaEst - \thetaY)] \le O(\smax(\CPos)^2) \le O\left(\dim \left[ \nY + 
\frac{M\nX}{\frac{\nX (M +\condX^2)\sminY^2}{\alphaY} + A} \right]^{-2}\right)\]
\end{proof}

The proof of Theorem~\ref{thm:lregression_bias_variance} is given by the combination of Lemma~\ref{lemma:variance_ubound} and Lemma~\ref{lemma:bias_ubound}, and the bias-variance decomposition of the risk .

%% file: main/appendices/experiment_details.tex
\section{Additional Experiment Details}
\label{app:exp_details}

\subsection{Hierarchical Bayes Evaluation}
We sample $M$ linear models according to the hierarchical model in Section~\ref{sec:hierarchical_bayes}, with design matrices constructed by uniformly sampling points, $x \sim U[-1,1]$, and storing the vector $\bx_j = x^j$, for $i=0,\ldots,d$ in each row of $\DXi{i}$.

To produce the plots in Figure~\ref{fig:hierarchical_lreg_simulation} we computed the average loss over 100 random draws of the training data and labels from the same set of fixed $\btheta_{1:M+1}$ values. The $\btheta$ values were sampled once from the hierarchical model with $\tau = [0, 1, 2, 0, 0, 3, 1]$, and $\sigma^2_\theta = 0.1$

\subsection{Sinusoid Regression with MAML}

\begin{table}[H]
    \centering
    \begin{tabular}{l|c}
    Hyper parameters & Description \\
    \hline
    $\sigma$ & noise at test time.\\
    M & number of tasks at the training tasks\\
    $M_q$     & number of tasks at the testing tasks \\
    eps\_per\_batch & episode per batch          \\
    train\_ampl\_range &    range of amplitude at training          \\
    train\_phase\_range &  range of phase at training \\
    val\_ampl\_range &  range of amplitude  at testing \\
    val\_phase\_range & range of phase at testing\\
    inner\_steps &  number of steps of Maml  \\
    inner\_lr &  learning rate used to optimize parameter of the model \\
    meta\_lr & used to optimize parameter of the meta-learner \\
    n & number of datapoints at training tasks(support set) \\
    k & number of datapoints at testing  tasks (support set)\\

    $n_q$ &  number of datapoints at training  tasks (query set)  .\\
    $k_q$ & number of datapoints at testing  tasks (query set).\\
    
    \end{tabular}
\end{table}

For all of these experiments we used a fully connected network with 6 layers and 40 hidden units per layer. The network is trained using the MAML algorithm \citep{finn2017model} with 5 inner steps using SGD with an inner learning rate of $10^{-3}$. We used Adam for the outer loop learning with a learning rate of $10^{-3}$.

The expected error was computed after 500 epochs of optimization and was averaged over 30 runs. We produced our results through a comprehensive grid search over 72 combinations of the settings below and it required around 30 minutes to produce the output of each setting, using a system with 1 gpu and 3 cpus. This experiment therefore lasted 20 hours in total.
\\
$M=50, n \in \{20, 200\} , k \in \{100,1000\},
\sigma \in [ 10^{-8}, 1.5], 
M_q = 100 , \text{ eps\_per\_batch} = 25, \text{ train\_ampl\_range} = [1,4] ,
\text{ train\_phase\_range} =[0, \pi / 2 ],
\text{ val\_ampl\_range} = [3,5],
\text{ val\_phase\_range}= [0, \pi / 2 ],
\text{ inner\_steps} =  5,
\text{ inner\_lr} = 10^{-3},
\text{ meta\_lr} = 10^{-3} $

%% file: main.bbl
\begin{thebibliography}{38}
\providecommand{\natexlab}[1]{#1}
\providecommand{\url}[1]{\texttt{#1}}
\expandafter\ifx\csname urlstyle\endcsname\relax
  \providecommand{\doi}[1]{doi: #1}\else
  \providecommand{\doi}{doi: \begingroup \urlstyle{rm}\Url}\fi

\bibitem[Amit and Meir(2017)]{amit2017meta}
Ron Amit and Ron Meir.
\newblock Meta-learning by adjusting priors based on extended {PAC}-bayes
  theory.
\newblock \emph{arXiv preprint arXiv:1711.01244}, 2017.

\bibitem[Andrychowicz et~al.(2016)Andrychowicz, Denil, Colmenarejo, Hoffman,
  Pfau, Schaul, and de~Freitas]{l2l}
Marcin Andrychowicz, Misha Denil, Sergio~Gomez Colmenarejo, Matthew~W. Hoffman,
  David Pfau, Tom Schaul, and Nando de~Freitas.
\newblock Learning to learn by gradient descent by gradient descent.
\newblock In \emph{Advances in Neural Information Processing Systems 29}, pages
  3981--3989, 2016.

\bibitem[Baxter(2000)]{baxter2000model}
Jonathan Baxter.
\newblock A model of inductive bias learning.
\newblock \emph{Journal of artificial intelligence research}, 12:\penalty0
  149--198, 2000.

\bibitem[Ben-David and Borbely(2008)]{ben2008notion}
Shai Ben-David and Reba~Schuller Borbely.
\newblock A notion of task relatedness yielding provable multiple-task learning
  guarantees.
\newblock \emph{Machine learning}, 73\penalty0 (3):\penalty0 273--287, 2008.

\bibitem[Ben-David et~al.(2010)Ben-David, Blitzer, Crammer, Kulesza, Pereira,
  and Vaughan]{ben2010theory}
Shai Ben-David, John Blitzer, Koby Crammer, Alex Kulesza, Fernando Pereira, and
  Jennifer~Wortman Vaughan.
\newblock A theory of learning from different domains.
\newblock \emph{Machine learning}, 79\penalty0 (1-2):\penalty0 151--175, 2010.

\bibitem[Bullins et~al.(2019)Bullins, Hazan, Kalai, and
  Livni]{pmlr-v98-bullins19a}
Brian Bullins, Elad Hazan, Adam Kalai, and Roi Livni.
\newblock Generalize across tasks: Efficient algorithms for linear
  representation learning.
\newblock In Aur\'elien Garivier and Satyen Kale, editors, \emph{Proceedings of
  the 30th International Conference on Algorithmic Learning Theory}, volume~98
  of \emph{Proceedings of Machine Learning Research}, pages 235--246, Chicago,
  Illinois, 22--24 Mar 2019. PMLR.
\newblock URL \url{http://proceedings.mlr.press/v98/bullins19a.html}.

\bibitem[Cao et~al.(2019)Cao, Law, and Fidler]{cao2019theoretical}
Tianshi Cao, Marc Law, and Sanja Fidler.
\newblock A theoretical analysis of the number of shots in few-shot learning.
\newblock \emph{arXiv preprint arXiv:1909.11722}, 2019.

\bibitem[Cover and Thomas(2012)]{cover2012elements}
Thomas~M Cover and Joy~A Thomas.
\newblock \emph{Elements of information theory}.
\newblock John Wiley \& Sons, 2012.

\bibitem[Denevi et~al.(2019)Denevi, Ciliberto, Grazzi, and
  Pontil]{pmlr-v97-denevi19a}
Giulia Denevi, Carlo Ciliberto, Riccardo Grazzi, and Massimiliano Pontil.
\newblock Learning-to-learn stochastic gradient descent with biased
  regularization.
\newblock In Kamalika Chaudhuri and Ruslan Salakhutdinov, editors,
  \emph{Proceedings of the 36th International Conference on Machine Learning},
  volume~97 of \emph{Proceedings of Machine Learning Research}, pages
  1566--1575, Long Beach, California, USA, 09--15 Jun 2019. PMLR.

\bibitem[Du et~al.(2020)Du, Hu, Kakade, Lee, and Lei]{du2020few}
Simon~S Du, Wei Hu, Sham~M Kakade, Jason~D Lee, and Qi~Lei.
\newblock Few-shot learning via learning the representation, provably.
\newblock \emph{arXiv preprint arXiv:2002.09434}, 2020.

\bibitem[Duan et~al.(2016)Duan, Schulman, Chen, Bartlett, Sutskever, and
  Abbeel]{duan2016rl2}
Yan Duan, John Schulman, Xi~Chen, Peter~L. Bartlett, Ilya Sutskever, and Pieter
  Abbeel.
\newblock Rl{\textdollar}{\^{}}2{\textdollar}: Fast reinforcement learning via
  slow reinforcement learning.
\newblock \emph{CoRR}, abs/1611.02779, 2016.

\bibitem[Fano(1961)]{fano1961transmission}
Robert Fano.
\newblock Transmission of information.
\newblock \emph{A Statistical Theory of Communication}, 1961.

\bibitem[Finn et~al.(2017)Finn, Abbeel, and Levine]{finn2017model}
Chelsea Finn, Pieter Abbeel, and Sergey Levine.
\newblock Model-agnostic meta-learning for fast adaptation of deep networks.
\newblock In \emph{Proceedings of the 34th International Conference on Machine
  Learning-Volume 70}, pages 1126--1135. JMLR. org, 2017.

\bibitem[Gelman et~al.(2013)Gelman, Carlin, Stern, Dunson, Vehtari, and
  Rubin]{gelman2013bayesian}
Andrew Gelman, John~B Carlin, Hal~S Stern, David~B Dunson, Aki Vehtari, and
  Donald~B Rubin.
\newblock \emph{Bayesian data analysis}.
\newblock Chapman and Hall/CRC, 2013.

\bibitem[Grant et~al.(2018)Grant, Finn, Levine, Darrell, and
  Griffiths]{grant2018recasting}
Erin Grant, Chelsea Finn, Sergey Levine, Trevor Darrell, and Thomas Griffiths.
\newblock Recasting gradient-based meta-learning as hierarchical {B}ayes.
\newblock \emph{arXiv preprint arXiv:1801.08930}, 2018.

\bibitem[Hanneke and Kpotufe(2019)]{NIPS2019_9179}
Steve Hanneke and Samory Kpotufe.
\newblock On the value of target data in transfer learning.
\newblock In \emph{Advances in Neural Information Processing Systems 32}, pages
  9871--9881. Curran Associates, Inc., 2019.

\bibitem[Hanneke and Kpotufe(2020)]{hanneke2020no}
Steve Hanneke and Samory Kpotufe.
\newblock A no-free-lunch theorem for multitask learning.
\newblock \emph{arXiv preprint arXiv:2006.15785}, 2020.

\bibitem[Jin et~al.(2009)Jin, Wang, and Zhou]{jin2009regularized}
Rong Jin, Shijun Wang, and Yang Zhou.
\newblock Regularized distance metric learning:theory and algorithm.
\newblock In \emph{Advances in Neural Information Processing Systems 22}, pages
  862--870, 2009.

\bibitem[Khas’~minskii(1979)]{khas1979lower}
Rafail~Z Khas’~minskii.
\newblock A lower bound on the risks of non-parametric estimates of densities
  in the uniform metric.
\newblock \emph{Theory of Probability \& Its Applications}, 23\penalty0
  (4):\penalty0 794--798, 1979.

\bibitem[Khodak et~al.(2019)Khodak, Balcan, and Talwalkar]{khodak2019provable}
Mikhail Khodak, Maria-Florina Balcan, and Ameet Talwalkar.
\newblock Provable guarantees for gradient-based meta-learning.
\newblock \emph{arXiv preprint arXiv:1902.10644}, 2019.

\bibitem[Kpotufe and Martinet(2018)]{pmlr-v75-kpotufe18a}
Samory Kpotufe and Guillaume Martinet.
\newblock Marginal singularity, and the benefits of labels in covariate-shift.
\newblock In S\'ebastien Bubeck, Vianney Perchet, and Philippe Rigollet,
  editors, \emph{Proceedings of the 31st Conference On Learning Theory},
  volume~75 of \emph{Proceedings of Machine Learning Research}, pages
  1882--1886. PMLR, 06--09 Jul 2018.

\bibitem[Loh(2017)]{loh2017lower}
Po-Ling Loh.
\newblock On lower bounds for statistical learning theory.
\newblock \emph{Entropy}, 19\penalty0 (11):\penalty0 617, 2017.

\bibitem[MacKay et~al.(2019)MacKay, Vicol, Lorraine, Duvenaud, and Grosse]{stn}
Matthew MacKay, Paul Vicol, Jonathan Lorraine, David Duvenaud, and Roger~B.
  Grosse.
\newblock Self-tuning networks: Bilevel optimization of hyperparameters using
  structured best-response functions.
\newblock In \emph{7th International Conference on Learning Representations},
  2019.

\bibitem[Maurer(2009)]{maurer2009transfer}
Andreas Maurer.
\newblock Transfer bounds for linear feature learning.
\newblock \emph{Machine learning}, 75\penalty0 (3):\penalty0 327--350, 2009.

\bibitem[Metz et~al.(2019)Metz, Maheswaranathan, Cheung, and
  Sohl{-}Dickstein]{metaunsup}
Luke Metz, Niru Maheswaranathan, Brian Cheung, and Jascha Sohl{-}Dickstein.
\newblock Meta-learning update rules for unsupervised representation learning.
\newblock In \emph{7th International Conference on Learning Representations},
  2019.

\bibitem[Mohri and Medina(2012)]{mohri2012new}
Mehryar Mohri and Andres~Munoz Medina.
\newblock New analysis and algorithm for learning with drifting distributions.
\newblock In \emph{International Conference on Algorithmic Learning Theory},
  pages 124--138. Springer, 2012.

\bibitem[Pentina and Lampert(2014)]{pentina2014pac}
Anastasia Pentina and Christoph Lampert.
\newblock A {PAC}-bayesian bound for lifelong learning.
\newblock In \emph{International Conference on Machine Learning}, pages
  991--999, 2014.

\bibitem[Raskutti et~al.(2011)Raskutti, Wainwright, and
  Yu]{raskutti2011minimax}
Garvesh Raskutti, Martin~J Wainwright, and Bin Yu.
\newblock Minimax rates of estimation for high-dimensional linear regression
  over $\ell\_q $-balls.
\newblock \emph{IEEE transactions on information theory}, 57\penalty0
  (10):\penalty0 6976--6994, 2011.

\bibitem[Rasmussen and Nickisch(2010)]{gpml}
Carl~Edward Rasmussen and Hannes Nickisch.
\newblock Gaussian processes for machine learning {(GPML)} toolbox.
\newblock \emph{J. Mach. Learn. Res.}, 11:\penalty0 3011--3015, 2010.

\bibitem[Ravi and Larochelle(2016)]{ravi2016optimization}
Sachin Ravi and Hugo Larochelle.
\newblock Optimization as a model for few-shot learning.
\newblock \emph{International Conference on Learning Representations}, 2016.

\bibitem[Robbins(1956)]{robbins1956}
Herbert Robbins.
\newblock An empirical bayes approach to statistics.
\newblock In \emph{Proceedings of the Third Berkeley Symposium on Mathematical
  Statistics and Probability, Volume 1: Contributions to the Theory of
  Statistics}, pages 157--163, Berkeley, Calif., 1956. University of California
  Press.

\bibitem[Saunshi et~al.(2019)Saunshi, Plevrakis, Arora, Khodak, and
  Khandeparkar]{saunshi2019theoretical}
Nikunj Saunshi, Orestis Plevrakis, Sanjeev Arora, Mikhail Khodak, and
  Hrishikesh Khandeparkar.
\newblock A theoretical analysis of contrastive unsupervised representation
  learning.
\newblock In \emph{International Conference on Machine Learning}, pages
  5628--5637, 2019.

\bibitem[Snell et~al.(2017)Snell, Swersky, and Zemel]{snell2017prototypical}
Jake Snell, Kevin Swersky, and Richard Zemel.
\newblock Prototypical networks for few-shot learning.
\newblock In \emph{Advances in Neural Information Processing Systems}, pages
  4077--4087, 2017.

\bibitem[Vilalta and Drissi(2002)]{vilalta2002perspective}
Ricardo Vilalta and Youssef Drissi.
\newblock A perspective view and survey of meta-learning.
\newblock \emph{Artificial intelligence review}, 18\penalty0 (2):\penalty0
  77--95, 2002.

\bibitem[Vinyals et~al.(2016)Vinyals, Blundell, Lillicrap, Wierstra,
  et~al.]{vinyals2016matching}
Oriol Vinyals, Charles Blundell, Timothy Lillicrap, Daan Wierstra, et~al.
\newblock Matching networks for one shot learning.
\newblock In \emph{Advances in neural information processing systems}, pages
  3630--3638, 2016.

\bibitem[Von~Neumann(1937)]{von1937some}
John Von~Neumann.
\newblock \emph{Some matrix-inequalities and metrization of matric space}.
\newblock 1937.

\bibitem[Wang et~al.(2019)Wang, Zhang, Liu, Shen, and
  Pineau]{wang2019multitask}
Boyu Wang, Hejia Zhang, Peng Liu, Zebang Shen, and Joelle Pineau.
\newblock Multitask metric learning: Theory and algorithm.
\newblock In Kamalika Chaudhuri and Masashi Sugiyama, editors,
  \emph{Proceedings of Machine Learning Research}, volume~89 of
  \emph{Proceedings of Machine Learning Research}, pages 3362--3371. PMLR,
  16--18 Apr 2019.

\bibitem[Yang and Barron(1999)]{yang1999information}
Yuhong Yang and Andrew Barron.
\newblock Information-theoretic determination of minimax rates of convergence.
\newblock \emph{Annals of Statistics}, pages 1564--1599, 1999.

\end{thebibliography}
